\documentclass[11pt]{article}

\usepackage{microtype}
\usepackage{graphicx}
\usepackage{subfigure}
\usepackage{booktabs} 
\usepackage{natbib}
\usepackage{epsfig,epsf,fancybox}
\usepackage{amsmath}
\usepackage{mathrsfs}
\usepackage{amssymb}
\usepackage{color}
\usepackage{multirow}
\usepackage{paralist}
\usepackage{verbatim}
\usepackage{algorithm}
\usepackage{algorithmic}
\usepackage{galois}
\usepackage{boxedminipage}
\usepackage{accents}
\usepackage{stmaryrd}
\usepackage[bottom]{footmisc}

\usepackage{natbib}

\usepackage{url}
\usepackage[colorlinks,linkcolor=magenta,citecolor=blue, pagebackref=true]{hyperref}
\renewcommand*{\backrefalt}[4]{%
    \ifcase #1 \footnotesize{(Not cited.)}%
    \or        \footnotesize{(Cited on page~#2.)}%
    \else      \footnotesize{(Cited on pages~#2.)}%
    \fi}

\newtheorem{theorem}{Theorem}[section]
\newtheorem{corollary}[theorem]{Corollary}
\newtheorem{lemma}[theorem]{Lemma}
\newtheorem{proposition}[theorem]{Proposition}

\newtheorem{definition}{Definition}[section]
\newtheorem{example}{Example}[section]

\newtheorem{assumption}[theorem]{Assumption}

\usepackage{hyperref}

\newcommand{\EE}{\mathbb{E}}

\newcommand{\PP}{\mathbb{P}}

\newcommand{\tr}{\textnormal{Trace}\,}

\newcommand{\diag}{\textnormal{diag}}

\newcommand{\x}{\mathbf x}
\newcommand{\y}{\mathbf y}
\newcommand{\cb}{\mathbf c}

\newcommand{\z}{\mathbf z}
\newcommand{\w}{\mathbf w}

\newcommand{\sv}{\mathbf v}
\newcommand{\argmin}{\mathop{\rm argmin}}

\newcommand{\HCal}{\mathcal{H}}

\newcommand{\SCal}{\mathcal{S}}

\newcommand{\WCal}{\mathcal{W}}

\newcommand{\zero}{\textbf{0}}

\newcommand{\br}{\mathbb{R}}

\newcommand{\ba}{\begin{array}}
\newcommand{\ea}{\end{array}}

\newcommand{\ZCal}{\mathcal{Z}}

\newcommand{\XCal}{\mathcal{X}}
\newcommand{\YCal}{\mathcal{Y}}

\usepackage{authblk}

\title{\bf{\LARGE{A Variational Inequality Approach to Bayesian Regression Games}}}

\author[1]{Wenshuo Guo\thanks{\{wguo, jordan, darren\_lin\}@cs.berkeley.edu. Authors are ordered alphabetically.}}
\author[1,2]{Michael I. Jordan}
\author[1]{Tianyi Lin}
\affil[1]{Department of Electrical Engineering and Computer Sciences, UC Berkeley}
\affil[2]{Department of Statistics, UC Berkeley}

\begin{document}









\maketitle

\begin{abstract} 
Bayesian regression games are a special class of two-player general-sum Bayesian games in which the learner is partially informed about the adversary's objective through a Bayesian prior. This formulation captures the uncertainty in regard to the adversary, and is useful in problems where the learner and adversary may have conflicting, but not necessarily perfectly antagonistic objectives. Although the Bayesian approach is a more general alternative to the standard minimax formulation, the applications of Bayesian regression games have been limited due to computational difficulties, and the existence and uniqueness of a Bayesian equilibrium are only known for quadratic cost functions. First, we prove the existence and uniqueness of a Bayesian equilibrium for a class of convex and smooth Bayesian games by regarding it as a solution of an infinite-dimensional variational inequality (VI) in Hilbert space. We consider two special cases in which the infinite-dimensional VI reduces to a high-dimensional VI or a nonconvex stochastic optimization, and provide two simple algorithms of solving them with strong convergence guarantees. Numerical results on real datasets demonstrate the promise of this approach.
\end{abstract}


\section{Introduction}
Adversarially robust models have seen a tremendous surge in research activity by various communities over the past decades, including statistics~\citep{Huber-2004-Robust}, optimization~\citep{Bental-2009-Robust}, and machine learning~\citep{Globerson-2006-Nightmare, Biggio-2018-Wild}. In machine learning, there has been renewed interest in the topic driven by work on adversarial methods in deep learning~\citep{Bruna-2014-Intriguing, Goodfellow-2015-Explaining}. There are two main considerations underlying this line of work: (1) real-world deployments of machine-learning methods require robustness to malicious data and it is an ongoing challenge to provide such robustness~\citep{Madry-2018-Towards, Wong-2020-Fast}; (2) adversarially robust models may generalize better~\citep{Zhu-2020-FreeLB} and have better interpretability properties than non-robust methods~\citep{Tsipras-2018-Robustness, Santurkar-2019-Image}. To this end, adversarially robust models are often preferred by practitioners in real-world applications.

From a game-theoretic point of view, adversarially robust models naturally form a two-player game between a learner and an adversary. Both players choose actions simultaneously, and their interacting dynamics constitute a \textit{noncooperative game}~\citep{Nash-1951-Non}. The learner's action is to select the best set of model parameters while maximizing the test accuracy; the adversary's action is to impose a perturbation on the input data distribution while paying a perturbation cost. Based on this general framework, existing approaches have mostly been restricted to the \textit{zero-sum} case,  in which the learner and adversary have fully conflicting goals. In this case, the learner is best off by choosing a \textit{minimax strategy}; i.e., minimizing the worst-case cost over the action space of the adversary. For classification and regression problems, properties of the minimax solutions have been derived under a variety of assumptions~\citep{Lanckriet-2002-Robust, El-2003-Robust, Globerson-2006-Nightmare, Sayed-2002-Uniqueness, Teo-2008-Convex}. 

On the positive side, the minimax solutions are computationally tractable in several specific settings~\citep{Sayed-2002-Uniqueness}, or can be approximated through a convex relaxation~\citep{Teo-2008-Convex}. However, they come with a few limitations. First, zero-sum games are not flexible enough to capture cases in which the learner and the adversary do not have perfectly antagonistic goals~\citep{Bruckner-2012-Static}. For example, a credit card defrauder’s goal of maximizing the illicit profit made from exploiting phished account information via spam emails is not the exact inverse of an email service provider's goal of achieving a close-to-zero false positives rate at spam recognition. In these cases, a minimax strategy can make overly pessimistic assumptions about the adversary's behavior and lead to an optimal outcome. 
One approach to filling this gap is by relaxing the zero-sum assumption to \textit{general-sum} games, in which the learner is fully aware of the costs of the adversary~\citep{Bruckner-2012-Static}. 

Moreover, in many practical applications, the cost function of the adversary is simply unknown to the learner. For instance, it is hard for an internet security service provider to know exactly the profit of the attackers. Therefore, the standard minimax solutions, which require \textit{complete information} of the game, become infeasible to compute.
This strong assumption of complete information can be lifted via a Bayesian game-theoretic framework~\citep{Harsanyi-1967-Games}. This has been pursued, for example, by~\citet{Grosshans-2013-Bayesian}, who propose \emph{Bayesian regression games}. In this class of games, the learner's uncertainty regarding the adversary's costs is reflected in a Bayesian prior over the parameters of the cost function. These authors derive sufficient conditions for the existence and uniqueness of a Bayesian equilibrium, as well as a graduated optimization algorithm to compute the equilibrium. 

Despite the appealing conceptual framework, the sufficient conditions for the equilibrium uniqueness for Bayesian regression games that are known to date~\citep[Theorem~2]{Grosshans-2013-Bayesian} are designed for quadratic cost functions. This excludes other common choices, e.g., logistic or smooth hinge functions, which are widely used in practice. Moreover, the algorithms studied in this line of work are heuristic, providing no theoretical guarantee of convergence. Thus we are motivated to tackle the following important open questions: \textit{Can we generalize existing sufficient conditions for the existence and uniqueness of a Bayesian equilibrium in Bayesian regression games to more general cost functions?} \emph{Can we develop efficient algorithms to compute the equilibrium?}

\paragraph{Contributions.} We present an affirmative answer to these questions in this paper. First, we prove sufficient conditions for the existence and uniqueness of a Bayesian equilibrium in a general class of convex and smooth Bayesian games using a variational inequality (VI) approach. In particular, we show that computing a Bayesian equilibrium amounts to solving an infinite-dimensional VI in a Hilbert space under certain conditions. This allows sufficient conditions to be derived using classical results from the optimization literature. Second, we consider two special settings with either a finite Bayesian prior or a quadratic adversarial loss and show that the infinite-dimensional VI reduces to a high-dimensional Euclidean VI or a nonconvex Euclidean stochastic optimization problem in these two settings. Third, we propose new algorithms to solve for the equilibrium with theoretical guarantees. The first algorithm uses the idea of projected reflected gradient with inertial extrapolation and achieves the strong convergence. The second algorithm uses the idea of randomized block coordinate descent, and is specialized to the finite Bayesian prior setting. We provide an analysis of its iteration complexity and convergence guarantee. Lastly, we empirically demonstrate the effectiveness of the proposed approach on real dataset.
  
\paragraph{Organization.} In Section~\ref{sec:related_work}, we overview  related work on Bayesian games and the computation of equilibria. In Section~\ref{sec:prelim}, we present background on Bayesian regression games and prove that the computation of a Bayesian equilibrium amounts to solving an infinite-dimensional VI. We also treat the existence and uniqueness of Bayesian equilibria and consider two special settings of a finite Bayesian prior and a quadratic adversarial loss. We propose specific algorithms and analyze their convergence and iteration complexity in Sections~\ref{sec:alg_inertial} and~\ref{sec:alg_rand}.  Numerical results demonstrating the favorable practical performance of our algorithms are presented in Section~\ref{sec:experiment}. We conclude in Section~\ref{sec:conclu} and provide all the missing proof details in the appendix. 

\paragraph{Notation.} We use bold lower-case letters such as $\x$ to denote vectors, upper-case letters such as $X$ to denote matrices, and calligraphic upper-case letters such as $\XCal$ to denote sets. The notation $[n]$ refers to $\{1, 2, \ldots, n\}$ for some integer $n > 0$. The symbols $\zero_n$ and $\zero_{n \times m}$ refer to the vector and the matrix in $\br^n$ and $\br^{n \times m}$ whose entries are all zeroes. We let $\EE[\cdot]$ denote an expectation and use $\EE_q[\cdot]$ to denote an expectation over a distribution $q$. For a differentiable function $f: \br^d \rightarrow \br$, we let $\nabla f(\x)$ denote the gradient of $f$ at $\x$. For a vector $\x \in \br^d$, we denote $\|\x\|$ as its $\ell_2$-norm. For a matrix $X \in \br^{d \times r}$, we let $\|X\|_F$ denote its Frobenius norm. As an abuse of notation, we denote $\langle\x, \y\rangle = x^\top y$ as the inner product between two vectors $\x, \y \in \br^d$ and $\langle X, Y\rangle = \tr(X^\top Y)$ as the inner product between two matrices $X, Y \in \br^{d \times r}$, where $\tr(\cdot)$ stands for the trace of a matrix. For $\XCal \subseteq \br^d$, we let $D_\XCal$ denote its diameter: $D_\XCal = \max_{\x, \x'\in \XCal} \|\x-\x'\|$. Given $\epsilon > 0$, $a=O(b(\epsilon))$ stands for the upper bound $a \leq C \cdot b(\epsilon)$, where $C>0$ is independent of $\epsilon$. Similarly, $a=\tilde{O}(b(\epsilon))$ indicates that the inequality may depend on a logarithmic function of $1/\epsilon$, where $C>0$ is independent of $\epsilon$.

\subsection{Related Work}\label{sec:related_work}
We refer to~\citet{Kolter-2018-Adversarial} as a reference point for the burgeoning literature on adversarially robust models. Despite the attention devoted to this topic,~\citet{Grosshans-2013-Bayesian} is one of very few papers that models the general-sum nature and the uncertainty in the adversary via the classical formalism of Bayesian games.

\paragraph{Bayesian games.} Bayesian game has been a classical approach in game theory to model the situation of asymmetric or incomplete information in games. In his seminal work,~\citet{Harsanyi-1967-Games} proved the existence of a \textit{Bayesian Nash equilibrium} in finite Bayesian games given a common prior and common knowledge of that prior among all the players. Subsequently,~\citet{Mertens-1985-Formulation} introduced a relaxed notion of a ``universal prior space'', which is a sufficiently large space that captures players’ higher-order beliefs. Further,~\citet{Aghassi-2006-Robust} considered the special case where the payoffs are drawn from a bounded uncertainty set but the distribution is fully unknown. 

More recently, the appealing formulation of Bayesian games which captures the players' uncertainty has led to many works in economics settings, in particular for auctions. Specifically, each bidder has a private valuation function that expresses complex preferences over all subsets of items, and bidders have beliefs about the valuation functions of the other bidders, in the form of probability distributions~\citep{Myerson-1985-Bayesian}. In this setting, a Bayesian equilibrium can be viewed as an approximation for the optimal social welfare value. Unfortunately, most existing results on the complexity of finding a Bayesian equilibrium in various auctions are negative~\cite{Christodoulou-2008-Bayesian, Bhawalkar-2011-Welfare, Feldman-2013-Simultaneous}. Computing a Bayesian equilibrium is in PP and even finding an $\epsilon$-approximate Bayesian equilibrium is NP-hard when $\epsilon>0$ is small~\citep{Cai-2014-Simultaneous}.

\paragraph{Equilibrium existence and computation.} The existence of a mixed strategy Nash equilibrium is well known in finite games with complete information~\citep{Nash-1950-Equilibrium}. Such existence results are derived via the Brouwder fixed-point theorem~\citep{Kakutani-1941-Generalization}, which suggests intuitively that a fixed-point iteration might be an efficient approach for computing a Nash equilibrium. Similar existence and uniqueness results are derived for concave games with complete information~\citep{Rosen-1965-Existence}. However, computation of the equilibrium has been a challenging problem.~\citet{Chen-2009-Settling} recently proved that the problem of finding a Nash equilibrium for even the simplest two-player general-sum games is PPAD-complete~\citep{Papadimitriou-1994-Complexity}. Further complexity results have been established for equilibrium computation in games with complete information under various assumptions~\citep{Gilboa-1989-Nash, Megiddo-1991-Total, Conitzer-2003-Complexity, Conitzer-2008-New, Daskalakis-2009-complexity, Rubinstein-2018-Inapproximability}. In Bayesian games, the complexity of deciding the existence of a pure Bayesian equilibrium is in general NP-hard~\citep{Conitzer-2003-Complexity, Gottlob-2007-Complexity}. Nonetheless, a few Bayesian games are computationally tractable. Two canonical examples are: (i) tree-games~\citep{Singh-2004-Computing}, where the cost function depends only on the actions/types of neighboring players and the interaction formed by the neighborhood relation is a tree; (ii) two-player zero-sum Bayesian games with finite prior. The counterfactual regret minimization (CFR) algorithm was proposed with solid theoretical guarantee~\citep{Zinkevich-2007-Regret}.

\section{Bayesian Regression Games}\label{sec:prelim}
In this section, we first present the setup and equilibrium concepts for Bayesian regression games with general cost functions. We then prove existence and uniqueness results by a variational inequality (VI) formulation. 

\subsection{Basic setup}
We consider a general-sum game between a \textit{learner} of a regression model and a \textit{data generator} who is able to perturb the data distribution. Let $(X, \y) \in \br^{n \times m} \times \br^n$ be a pair of data matrix and target vector, which are generated by the data generator at training time. Denote each row of $X$ as $\x_i$, each entry of $\y$ as $y_i$ for $i \in [n]$. We assume that all pairs of instances $\{(\x_i, y_i)\}_{i=1}^n$ are drawn from an unknown distribution $\mu$ over $\XCal \times \YCal$. At testing time, the data generator provides new instances drawn from another distribution $\bar{\mu}$ defined on $\XCal \times \YCal$, but might not be $\mu$. These instances are generally not available at the training time. 

For the learner, we denote the action space as $\WCal \subseteq \br^m$, which is the space of the regression parameters. We denote an instance-specific weight by $c_l(\x, y) \geq 0$. Then, the learner's cost at testing time is the weighted average loss, i.e. $\theta_l(\w, \bar{\mu}, c_l) = \int c_l(\x, y)f_l(\w, \x, y) \; d\bar{\mu}(\x, y)$, where $\w \in \WCal$ is a parameter of the prediction model and $f_l$ is the cost function.  

For the data generator, intuitively, the goal is to manipulate the input data in order to achieve certain targeted predictions. Therefore, the costs of the data generator contain two parts. The first part is a cost of performing the manipulation, and the second part is a cost which quantifies the loss between the actual predictions and the data generator's targeted predictions. For the cost of manipulation, the manipulation of the input data is reflected in the difference between the distributions $\mu$ and $\bar{\mu}$.  We denote the cost of such manipulation
for the data generator as $\Omega_d(\mu, \bar{\mu})$. For the cost on the predictions, we denote the vector of target values for $n$ data points as $z(\x, y) \in \ZCal \subseteq \br^n$, and $f_d$ as the cost function. Similarly to the learner, we also allow the data generator to have a instance-specific weight $c_d(\x, y)$. Then, the data generator's costs are defined by $\theta_d(\w, \bar{\mu}, c_d) = \int c_d(\x, y)f_d(\w, \x, z(\x, y)) \; d\bar{\mu}(\x, y) + \Omega_d(\mu, \bar{\mu})$. 


Note that the theoretical costs of both players depend on the unknown distributions $\mu$ and $\bar{\mu}$. Thus, we focus on the regularized empirical counterparts of the theoretical costs based on the training samples $(X, \y, \z)$, where $\y \in \br^n$ and $\z \in \br^n$ are the empirical versions of $y$  and $z(\x, y)$ respectively. The empirical counterpart of the term $\Omega_d(\mu, \bar{\mu})$ is represented by the difference between the training matrix $X$ and a perturbed matrix $\bar{X}$ that would be the outcome of applying the transformation which translates $\mu$ into $\bar{\mu}$ to $X$. For simplicity, we denote this by $\Omega_d(X, \bar{X})$. We also denote the empirical instance-specific weights for the two players by $\cb_l \in \br^n$ and $\cb_d \in \br^n$. Then, the empirical costs of the learner and the data generator are given by 
\begin{equation}\label{Def:player_cost}
\begin{array}{l}
\widehat{\theta}_l(\w, \bar{X}, \cb_l) = \sum_{i=1}^n c_{l,i}f_l(\w, \bar{\x}_i, y_i) + \Omega_l(\w), \\ 
\widehat{\theta}_d(\w, \bar{X}, \cb_d) = \sum_{i=1}^n c_{d,i}f_d(\w, \bar{\x}_i, z_i) + \Omega_d(X, \bar{X}),
\end{array} 
\end{equation}
where $\Omega_l(\w)$ is a regularization term. With a focus on machine-learning applications, we illustrate the above setup with a few common loss functions, regularization terms and constraint sets. Besides the quadratic case~\citep{Grosshans-2013-Bayesian}, we provide a motivating example using logistic function and the unit ball constraint sets. This is encouraged by the fact that the logistic loss functions are more suitable than quadratic ones for the binary classification/regression problems, such as the email spam filtering and network security detection.   
\begin{example}[quadratic] $f_l(\w, \x, y)=(\x^\top\w-y)^2$, $\Omega_l(\w)=\|\w\|^2$, $f_d(\w, \x, z)=(\x^\top\w-z)^2$ and $\Omega_d(X, \bar{X})=\|X-\bar{X}\|_F^2$; $\WCal=\br^m$, $\XCal=\br^{n \times m}$ and $\YCal=\ZCal=\br^n$. 
\end{example}
\begin{example}[logistic] $f_l(\w, \x, y)=\log(1+\exp(-y\x^\top\w))$, $\Omega_l(\w)=\|\w\|^2$, $f_d(\w, \x, z)=\log(1+\exp(-z\x^\top\w))$ and $\Omega_d(X, \bar{X})=\|X-\bar{X}\|_F^2$; $\WCal=\{\w \in \br^m \mid \|\w\| \leq 1\}$, $\XCal=\{X \in \br^{n \times m} \mid \|X\|_F \leq 1\}$ and $\YCal=\ZCal=\{-1, 1\}^n$. 
\end{example}

\subsection{Bayesian regression game}
The previous basic setup describes a general-sum two-player regression game. Now we present the formulation of the Bayesian regression game based on it. First, note that the cost functions (see Eq~\eqref{Def:player_cost}) depend on the actions of both players---the parameters $\w$ and the transformation manifested in $\bar{X}$. Further, the cost function of the data generator depends on the instance-specific weight $\cb_d$. We are interested in a setting where these weights are private information of the data generator, and are unknown to the learner. Instead, the learner is only informed about the instance-specific weight $\cb_d$ through a Bayesian prior $q(\cb_d)$. As argued by~\citet{Grosshans-2013-Bayesian}, this asymmetry of uncertainty is crucial. By modeling the learner's lack of information about the data generator, we intend to make the learner more robust to new adversarial examples. Thus, this setting is naturally formulated as a \textit{two-player general-sum Bayesian game}.  

We denote this Bayesian regression game by the tuple $G = (\WCal, \Sigma, \widehat{\theta}_l, \widehat{\theta}_d, \cb_l, q)$. From the learner's viewpoint, $\cb_d$ is a random variable that is drawn from a Bayesian prior $q$ at testing time. During training, the data generator commits to a parametric strategy, $\sigma: \br^n \rightarrow \XCal$, which maps a value of $\cb_d$ (unknown to the learner) to a transformation reflected in $\bar{X}$. In other words, the action space of the data generator $\Sigma$ contains all the functions from $\br^n$ to $\XCal$. We also define Bayesian Equilibrium and its approximation. We denote the best responses of the learner and the data generator as: $\w^\star[\sigma] = \argmin_{\w \in \WCal} \EE_q[\widehat{\theta}_l(\w, \sigma(\cb_d), \cb_l)]$ and $\sigma^\star[\w](\cb_d) = \argmin_{\bar{X} \in \XCal} \widehat{\theta}_d(\w, \bar{X}, \cb_d)$.  
\begin{definition}[Bayesian Equilibrium]
The strategy profile $(\w_\star, \sigma_\star) \in \WCal \times \Sigma$ is a \emph{Bayesian equilibrium} for the Bayesian regression game $G = (\WCal, \Sigma, \widehat{\theta}_l, \widehat{\theta}_d, \cb_l, q)$ if, for a.e., $\omega \in \Omega$, the following statement holds true: $(\w_\star, \sigma_\star(\cb_d)) = (\w^\star[\sigma_\star], \sigma^\star[\w_\star](\cb_d))$, or equivalently, $\EE_q[\widehat{\theta}_l(\w_\star, \sigma_\star(\cb_d), \cb_l)] \leq \EE_q[\widehat{\theta}_l(\w, \sigma_\star(\cb_d), \cb_l)]$ for all $\w \in \WCal$ and $\widehat{\theta}_d(\w_\star, \sigma_\star(\cb_d), \cb_d) \leq \widehat{\theta}_d(\w_\star, \bar{X}, \cb_d)$ for all $\bar{X} \in \XCal$. 
\end{definition}   
\begin{definition}[$\epsilon$-Bayesian Equilibrium]
The strategy profile $(\w, \sigma) \in \WCal \times \Sigma$ is an \emph{$\epsilon$-Bayesian equilibrium} for the Bayesian regression game $G = (\WCal, \Sigma, \widehat{\theta}_l, \widehat{\theta}_d, \cb_l, q)$ if, for a.e., $\omega \in \Omega$, the following statement holds true: $\|\w - \w_\star\|^2 + \EE_q[\|\sigma(\cb_d) - \sigma_\star(\cb_d)\|_F^2] \leq \epsilon$, where the strategy profile $(\w_\star, \sigma_\star)$ is a Bayesian equilibrium. 
\end{definition}

When the distribution $q$ is a point mass, it is clear that no uncertainty exists and a Bayesian equilibrium is a Nash equilibrium. This corresponds to a two-player general-sum game which is still more general than the minimax strategy.

\subsection{Infinite-dimensional variational inequality model}
We show that we can regard a Bayesian equilibrium as a solution of an infinite-dimensional VI~\citep{Kinderlehrer-2000-Introduction} (Eq.~\eqref{prob:VI-infinite}). This characterization is not only necessary but sufficient under certain assumptions. We adapt a proof technique in~\citet{Ui-2016-Bayesian}, and specialized the VI model to Bayesian regression games. We make the following assumption throughout this paper.   
\begin{assumption}\label{Assumption:main}
The Bayesian regression game $G = (\WCal, \Sigma, \widehat{\theta}_l, \widehat{\theta}_d, \cb_l, q)$ is convex and smooth: \textbf{(1)} The cost function $\widehat{\theta}_l(\cdot, \bar{X}, \cb_l): \WCal \rightarrow \br$ is convex and continuously differentiable for each $\bar{X} \in \XCal$, and $\|\nabla_\w\widehat{\theta}_l(\w, \bar{X}, \cb_l)\|^2 < +\infty$ for each $(\w, \bar{X}) \in \WCal \times \XCal$; \textbf{(2)} The cost function $\widehat{\theta}_d(\w, \cdot, \cb_d): \XCal \rightarrow \br$ is convex and continuously differentiable for each $\w \in \WCal$, and $\EE_q[\|\nabla_{\bar{X}}\widehat{\theta}_d(\w, \bar{X}, \cb_d)\|_F^2] < +\infty$ for each $(\w, \bar{X}) \in \WCal \times \XCal$; \textbf{(3)} The action spaces $\WCal$ and $\XCal$ are both closed and convex.
\end{assumption} 
We provide some intuitions for these assumptions. First, note that a Bayesian equilibrium is characterized by the coordinate-wise minimization of the cost functions $\widehat{\theta}_l$ and $\widehat{\theta}_d$ over the action spaces $\WCal$ and $\XCal$. Thus, it is necessary to impose convexity conditions on the cost functions and a few moment conditions on the Bayesian prior $q$. Furthermore, our moment conditions are implied by the assumptions $\EE_q[c_{d,i}] < +\infty$ for all $i \in [n]$ made in~\citet[Theorem~1 and~2]{Grosshans-2013-Bayesian} and are thus slightly weaker. Finally, even if the Bayesian regression game $G = (\WCal, \Sigma, \widehat{\theta}_l, \widehat{\theta}_d, \cb_l, q)$ is not smooth, we can derive a similar first-order condition using the subgradients of $\widehat{\theta}_l(\cdot, \bar{X}, \cb_l)$ and $\widehat{\theta}_d(\w, \cdot, \cb_d)$ and regard a Bayesian equilibrium as a solution of a multi-valued VI~\citep{Kinderlehrer-2000-Introduction}; see Eq.~\eqref{prob:VI-infinite} for the details. 
\begin{lemma}\label{lemma:VI}
Under Assumption~\ref{Assumption:main}, the strategy profile $(\w_\star, \sigma_\star) \in \WCal \times \Sigma$ is a Bayesian equilibrium for the Bayesian regression game $G = (\WCal, \Sigma, \widehat{\theta}_l, \widehat{\theta}_d, \cb_l, q)$ if and only if, for a.e. $\omega \in \Omega$, the following statement holds true:  
\begin{equation}\label{lemma:VI-main}
\begin{array}{lcll}
\EE_q[\langle\w-\w_\star, \nabla_\w\widehat{\theta}_l(\w_\star, \sigma_\star(\cb_d), \cb_l)\rangle] & \geq & 0, & \forall \w \in \WCal, \\
\langle\bar{X} - \sigma_\star(\cb_d), \nabla_{\bar{X}}\widehat{\theta}_d(\w_\star, \sigma_\star(\cb_d), \cb_d)\rangle & \geq & 0, & \forall \bar{X} \in \XCal.  
\end{array}
\end{equation}
\end{lemma}
\begin{theorem}\label{Theorem:VI}
Under Assumption~\ref{Assumption:main}, the strategy profile $(\w_\star, \sigma_\star) \in \WCal \times \Sigma$ is a Bayesian equilibrium for the Bayesian regression game $G = (\WCal, \Sigma, \widehat{\theta}_l, \widehat{\theta}_d, \cb_l, q)$ iff for $\forall (\w, \sigma) \in \WCal \times \Sigma$, we have
\begin{equation}\label{prob:VI-infinite}
\EE_q[\langle \w-\w_\star, \nabla_\w\widehat{\theta}_l(\w_\star, \sigma_\star(\cb_d), \cb_l)\rangle + \langle \sigma(\cb_d) - \sigma_\star(\cb_d), \nabla_{\bar{X}}\widehat{\theta}_d(\w_\star, \sigma_\star(\cb_d), \cb_d)\rangle] \geq 0. 
\end{equation}
\end{theorem}
For a game with complete information, it is common to study the existence, uniqueness and computation of a Nash equilibrium by regarding it as a solution of a VI in a finite-dimensional space. This approach dates back to~\citet{Lions-1967-Variational} and has been thoroughly studied in the optimization literature~\citep{Facchinei-2007-Finite}. While the VI approach can be formally extended to an infinite-dimensional space~\citep{Kinderlehrer-2000-Introduction}, it has not been recognized as a useful analytical tool to study games with incomplete information. The only work that we are aware of in this vein is~\citet{Ui-2016-Bayesian}, who gives a sufficient condition for the existence and uniqueness of a Bayesian equilibrium by regarding it as a solution of an infinite-dimensional VI. The focus in that work is, however, a general setting without any consideration of the  computation of an equilibrium. 

\subsection{Equilibrium existence and uniqueness}\label{subsec:equilibrium}
Based on the infinite-dimensional VI formulation (see Eq.~\eqref{prob:VI-infinite}), we prove a set of sufficient conditions for the existence and uniqueness of a Bayesian equilibrium. Our results generalize the existing work to Bayesian regression games with general convex cost functions.    

For the Bayesian regression game $G$ with its equilibrium defined by Eq.~\eqref{prob:VI-infinite}, we define a Hilbert space $\HCal$ consisting of (an equivalence class of) functions $\beta: \br^n \mapsto \br^m \times \br^{n \times m}$ with the inner product by $\langle (\w, \sigma), (\w', \sigma')\rangle_\HCal = \EE_q[\langle\w(\cb_d), \w'(\cb_d)\rangle + \langle \sigma(\cb_d), \sigma'(\cb_d)\rangle] < +\infty$. Note that each element in $\WCal \subseteq \br^m$ can be regarded as a constant function from $\br^n$ to $\br^m$ whose value is this element. We denote the set of these constant functions by $\Sigma_\WCal$ (an equivalence class of $\WCal$) and define a mapping $T: \Sigma_\WCal \times \Sigma \rightarrow \HCal$ by  
\begin{equation}\label{Eq:mapping-infinite-main}
T\begin{pmatrix} \w \\ \sigma(\cdot) \end{pmatrix} \ = \ \begin{pmatrix} \nabla_\w\widehat{\theta}_l(\w, \sigma(\cdot), \cb_l) \\ \nabla_{\bar{X}}\widehat{\theta}_d(\w, \sigma(\cdot), \cdot) \end{pmatrix} \ \in \ \HCal.
\end{equation}
Thus, the computation of a Bayesian equilibrium is equivalent to solving a VI in the space $\HCal$. This allows us to analyze the existence and uniqueness of a Bayesian equilibrium under the VI framework. For example, the existence of a Bayesian equilibrium is guaranteed by the continuity and monotonicity of $T$ as well as some additional conditions on $\WCal \times \Sigma$. 
\begin{definition}
Let $\HCal$ be a Hilbert space with the inner product $\langle \cdot, \cdot\rangle_\HCal$, we define $\SCal \subseteq \HCal$ as a closed and convex set and $T:\SCal \rightarrow \HCal$ as a mapping. Then, $T$ is \textbf{monotone} if $\langle T\beta-T\beta', \beta-\beta'\rangle_\HCal \geq 0$ for each $\beta, \beta' \in \SCal$; $T$ is \textbf{strictly monotone} if $\langle T\beta-T\beta', \beta-\beta'\rangle_\HCal > 0$ for each $\beta, \beta' \in \SCal$ with $\beta \neq \beta'$; $T$ is \textbf{$\lambda$-strongly monotone} ($\lambda > 0$) if $\langle T\beta-T\beta', \beta-\beta'\rangle_\HCal \geq \lambda\|\beta-\beta'\|_\HCal^2$ for each $\beta, \beta' \in \SCal$.
\end{definition} 
We summarize the existence and uniqueness results in the following two theorems. 
\begin{theorem}[Existence]\label{Theorem:VI-existence}
Suppose that the mapping $T$ defined by Eq.~\eqref{Eq:mapping-infinite-main} is continuous and monotone, and the action space $\WCal \times \Sigma$ is nonempty, closed and convex. If $\WCal \times \Sigma$ is compact, or there exists $(\w_0, \sigma_0) \in \WCal \times \Sigma$ such that, for all $(\w, \sigma) \in \WCal \times \Sigma$ satisfying $\|\w\|^2+\EE_q[\|\sigma(\cb_d)\|_F^2] \rightarrow +\infty$, the following statement holds true:  
\begin{equation}\label{Condition:VI-existence}
\frac{\EE_q[\langle \w-\w_0, \nabla_\w\widehat{\theta}_l(\w, \sigma(\cb_d), \cb_l)\rangle + \langle \sigma(\cb_d) - \sigma_0(\cb_d), \nabla_{\bar{X}}\widehat{\theta}_d(\w, \sigma(\cb_d), \cb_d)\rangle]}{\sqrt{\|\w\|^2+\EE_q[\|\sigma(\cb_d)\|_F^2]}} \rightarrow +\infty.
\end{equation}
Then, the VI in Eq.~\eqref{prob:VI-infinite} has at least one solution.  
\end{theorem}
\begin{corollary}\label{corollary:VI-existence-first}
Under Assumption~\ref{Assumption:main}, if $T$ defined by Eq.~\eqref{Eq:mapping-infinite-main} is monotone and there exists $(\w_0, \sigma_0) \in \WCal \times \Sigma$ such that, for all $(\w, \sigma) \in \WCal \times \Sigma$ satisfying $\|\w\|^2+\EE_q[\|\sigma(\cb_d)\|_F^2] \rightarrow +\infty$, Eq.~\eqref{Condition:VI-existence} holds. Then, the Bayesian regression game $G$ has at least one Bayesian equilibrium. 
\end{corollary}
\begin{corollary}\label{corollary:VI-existence-second}
Under Assumption~\ref{Assumption:main}, we assume that $T$ defined by Eq.~\eqref{Eq:mapping-infinite-main} is monotone and the action space $\WCal \times \Sigma$ is compact. Then, the Bayesian regression game $G$ has at least one Bayesian equilibrium. 
\end{corollary} 
\begin{theorem}[Existence and Uniqueness]\label{Theorem:existence-uniqueness}
Suppose that $T$ defined by Eq.~\eqref{Eq:mapping-infinite-main} is continuous and strictly monotone, and $\WCal \times \Sigma$ is nonempty, closed and convex. If $\WCal \times \Sigma$ is compact, or there exists $(\w_0, \sigma_0) \in \WCal \times \Sigma$ such that, for all $(\w, \sigma) \in \WCal \times \Sigma$ satisfying $\|\w\|^2+\EE_q[\|\sigma(\cb_d)\|_F^2] \rightarrow +\infty$, Eq.~\eqref{Condition:VI-existence} holds true. The VI in Eq.~\eqref{prob:VI-infinite} has a unique solution.
\end{theorem}
\begin{corollary}\label{corollary:VI-existence-third}
Under Assumption~\ref{Assumption:main}, if $T$ defined by Eq.~\eqref{Eq:mapping-infinite-main} is $\lambda$-strongly monotone. Then, the Bayesian regression game $G$ has a unique Bayesian equilibrium. 
\end{corollary}
\begin{corollary}\label{corollary:VI-existence-fourth}
Under Assumption~\ref{Assumption:main}, if the mapping $T$ defined by Eq.~\eqref{Eq:mapping-infinite-main} is strictly monotone and the action space $\WCal \times \Sigma$ is compact. Then, the Bayesian regression game $G$ has a unique Bayesian equilibrium. 
\end{corollary}
The monotonicity condition in Corollary~\ref{corollary:VI-existence-fourth} generalizes~\cite[Theorem~2]{Grosshans-2013-Bayesian}, which is a special case with quadratic loss functions. Our VI approach also supplies an intuitive yet rigorous justification for Eq.~(\ref{Condition:VI-existence}), demonstrating that it arises from the strict monotonicity of the mapping $T$.

\subsection{Two special settings}
For practical purposes, we consider two special settings where the computation of a Bayesian equilibrium reduces to solving a high-dimensional VI or solving a nonconvex stochastic optimization problem, both in Euclidean space. 

\paragraph{Case I: Finite Bayesian prior.} Let $K > 0$ be an integer, assume that the Bayesian prior $q$ is a distribution with support $\{\sv_1, \ldots, \sv_K\}$; i.e., $q(\cb_d = \sv_k) = p_k > 0$ for all $k \in [K]$, $\sum_{i=1}^K p_k = 1$. Then, the VI in Eq.~\eqref{prob:VI-infinite} becomes 
\begin{equation}\label{prob:VI-infinite-finite}
\sum_{k=1}^K p_k[\langle \w-\w_\star, \nabla_\w\widehat{\theta}_l(\w_\star, \sigma_\star(\sv_k), \cb_l)\rangle + \langle \sigma(\sv_k) - \sigma_\star(\sv_k), \nabla_{\bar{X}}\widehat{\theta}_d(\w_\star, \sigma_\star(\sv_k), \sv_k)\rangle] \ \geq \ 0,
\end{equation}
for all $(\w, \sigma) \in \WCal \times \Sigma$. By definition, $\sigma, \sigma_\star \in \Sigma$ are both mappings from $\br^n$ to $\XCal$ in an infinite-dimensional space. When the Bayesian prior is finite, Eq.~\eqref{prob:VI-finite} implies that $\sigma$ can be fully represented by $(\sigma(\sv_1), \sigma(\sv_2), \ldots, \sigma(\sv_K))$ where the range $\{\sv_1, \sv_2, \ldots, \sv_K\}$ is finite and known. For simplicity, define $\sigma^k = \sigma(\sv_k)$ and $\sigma_\star^k = \sigma_\star(\sv_k)$ for all $k \in [K]$. Then, the VI in Eq.~\eqref{prob:VI-infinite-finite} can be reformulated as follows:
\begin{equation}\label{prob:VI-finite}
\sum_{k=1}^K p_k[\langle \w-\w_\star, \nabla_\w\widehat{\theta}_l(\w_\star, \sigma_\star^k, \cb_l)\rangle + \langle \sigma^k - \sigma_\star^k, \nabla_{\bar{X}}\widehat{\theta}_d(\w_\star, \sigma_\star^k, \sv_k)\rangle] \ \geq \ 0,
\end{equation}
for all $\w \in \WCal$ and $\sigma^k \in \XCal$ for all $k \in [K]$. Note that the VI in Eq.~\eqref{prob:VI-finite} is a high-dimensional VI in Euclidean space. Indeed, $\w \in \br^m$ and $\sigma^k \in \br^{n \times m}$ for all $k \in [K]$, implying that the total number of unknown variables is $m+mnK$.

\paragraph{Case II: Quadratic adversarial loss.} Consider $f_d(\w, \x, z)=(\x^\top\w-z)^2$, $\Omega_d(X, \bar{X})=\|X-\bar{X}\|_F^2$, where $\x \in \br^{m}$, and $X, \bar X \in \XCal=\br^{n \times m}$, we have $\sigma^\star[\w](\cb_d) = \argmin_{\bar{X} \in \br^{n \times m}} \sum_{i=1}^n c_{d,i}(\bar{\x}_i^\top\w-z_i)^2 + \|X-\bar{X}\|_F^2$. By~\cite[Lemma~1]{Grosshans-2013-Bayesian}, $\sigma^\star[\w](\cb_d) \ = \ X - (\diag(\cb_d)^{-1} + \|\w\|^2 I_n)^{-1}(X\w-\z)\w^\top$. Equivalently, we have 
\begin{equation}
[\sigma^\star[\w](\cb_d)]_i = \x_i - \frac{c_{d,i}(\x_i^\top\w - z_i)}{1+\|\w\|^2c_{d,i}}\w \in \br^m.
\end{equation}
Putting these pieces together with the best response of the learner yields the following stochastic optimization problem:
\begin{equation}\label{prob:opt-stochastic}
\min_{\w \in \WCal} \ \EE_q\left[\sum_{i=1}^n c_{l,i}f_l\left(\w, \x_i - \frac{c_{d,i}(\x_i^\top\w - z_i)}{1+\|\w\|^2c_{d,i}}\w, y_i\right)\right] + \Omega_l(\w).
\end{equation}
The computation of a Bayesian equilibrium thus reduces to the solution of a nonconvex stochastic optimization problem. Standard gradient descent approaches can not be applied for solving the optimization problem in Eq.~\eqref{prob:opt-stochastic} since the integration with respect to a Bayesian prior $q$ does not have a closed-form expression in general. Nonetheless, the Bayesian prior $q$ is known and accessible through drawing samples, a stochastic-gradient-based algorithm can be applied and particular examples include stochastic gradient descent (SGD)~\citep{Robbins-1951-Stochastic, Bottou-1998-Online} or its adaptive variants AdaGrad and ADAM~\citep{Duchi-2011-Adaptive, Kingma-2015-Adam}. 

\section{Projected Reflected Gradient with Inertial Extrapolation}\label{sec:alg_inertial}
We present the \textit{projected reflected gradient with inertial extrapolation} (PRG-IE) algorithm for solving Eq.~\eqref{prob:VI-infinite}. We make the following assumption throughout this section and we discuss the intuitions of them afterwards.
\begin{algorithm}[!t]
\caption{Projected Reflected Gradient with Inertial Extrapolation (PRG-IE)}\label{Algorithm:PRG-IE}
\begin{algorithmic}[1]
\STATE \textbf{Input:} smoothness parameter $L > 0$; Bayesian prior $q$; weight $\cb_l \in \br^n$.  
\STATE \textbf{Initialize:} $\widetilde{\w}_0, \widetilde{\w}_1, \w_1 \in \WCal$, $\widetilde{\sigma}_0, \widetilde{\sigma}_1, \sigma_1 \in \Sigma$ and $0 < \gamma < \min\{1, \frac{1}{100L}\}$.  
\FOR{$t=1,2,\ldots,T$}
\STATE Compute $\delta_t \leftarrow 1/t$. 
\STATE Compute $(\widetilde{\w}_{t+1}, \widetilde{\sigma}_{t+1})$ by
\begin{align*}
\widetilde{\w}_{t+1} & \leftarrow P_\WCal(\w_t - \gamma\nabla_\w\widehat{\theta}_l(2\widetilde{\w}_t-\widetilde{\w}_{t-1}, 2\widetilde{\sigma}_t-\widetilde{\sigma}_{t-1}, \cb_l)), \\
\widetilde{\sigma}_{t+1} & \leftarrow P_\Sigma(\sigma_t - \gamma\nabla_{\bar{X}}\widehat{\theta}_d(2\widetilde{\w}_t-\widetilde{\w}_{t-1}, 2\widetilde{\sigma}_t-\widetilde{\sigma}_{t-1}, \cb_d)). 
\end{align*} 
\STATE Compute $(\w_{t+1}, \sigma_{t+1})$ by
\begin{equation*}
\begin{pmatrix} \w_{t+1} \\ \sigma_{t+1} \end{pmatrix} \ \leftarrow \ \frac{\delta_t}{2}\begin{pmatrix} \w_t \\ \sigma_t \end{pmatrix} + \left(1-\delta_t\right)\begin{pmatrix} \widetilde{\w}_{t+1} \\ \widetilde{\sigma}_{t+1} \end{pmatrix}. 
\end{equation*}
\ENDFOR
\end{algorithmic}
\end{algorithm}
\begin{assumption}\label{Assumption:PRG-IE}
The Bayesian regression game $G$ satisfies the following: (i) $\WCal$ and $\XCal$ are both compact with $\zero_m \in \WCal$ and $\zero_{n \times m} \in \XCal$; (ii) Given that $\cb_l \in \br^n$ is fixed, we have
\begin{eqnarray*}
& & \lefteqn{\EE_q[\langle\nabla_\w\widehat{\theta}_l(\w, \sigma(\cb_d), \cb_l) - \nabla_\w\widehat{\theta}_l(\w', \sigma'(\cb_d), \cb_l), \w - \w'\rangle} \\ 
& & + \langle \nabla_{\bar{X}}\widehat{\theta}_d(\w, \sigma(\cb_d), \cb_d) - \nabla_{\bar{X}}\widehat{\theta}_d(\w', \sigma'(\cb_d), \cb_d), \sigma(\cb_d) - \sigma'(\cb_d)\rangle] \\ 
& & > 0, \qquad \qquad \forall (\w, \sigma) \neq (\w', \sigma'),  
\end{eqnarray*}
(iii) There exists a constant $L>0$ such that $\nabla_\w\widehat{\theta}_l(\cdot, \cdot, \cb_l): \WCal \times \XCal \rightarrow \br^m$ is $L$-Lipschitz for each $\cb_l \in \br^n$ and $\nabla_{\bar{X}}\widehat{\theta}_d(\cdot, \cdot, \cb_d): \WCal \times \XCal \rightarrow \br^{n \times m}$ is $L$-Lipschitz for each $\cb_d \in \br^n$, i.e. for each $(\w, \bar{X}), (\w', \bar{X}') \in \WCal \times \XCal$, we have 
\begin{eqnarray*}
\|\nabla_\w\widehat{\theta}_l(\w, \bar{X}, \cb_l) - \nabla_\w\widehat{\theta}_l(\w', \bar{X}', \cb_l)\| & \leq & L(\|\w-\w'\| + \|\bar{X}-\bar{X}'\|_F), \\
\|\nabla_{\bar{X}}\widehat{\theta}_d(\w, \bar{X}, \cb_d) - \nabla_{\bar{X}}\widehat{\theta}_d(\w', \bar{X}', \cb_d)\|_F & \leq & L(\|\w-\w'\| + \|\bar{X}-\bar{X}'\|_F). 
\end{eqnarray*}
\end{assumption} 
Assumption~\ref{Assumption:PRG-IE} is standard in optimization and game theory. The second condition can be interpreted as the strict monotonicity of $T$ defined by Eq.~\eqref{Eq:mapping-infinite-main} in terms of the inner product $\langle \cdot, \cdot\rangle_\HCal: \HCal \times \HCal \rightarrow \br$. The third condition imposes a weak Lipschitz condition on $\nabla_\w\widehat{\theta}_l(\cdot, \cdot, \cb_l)$ and $\nabla_{\bar{X}}\widehat{\theta}_d(\cdot, \cdot, \cb_d)$. The third condition is necessary for the existence of at least one Bayesian equilibrium and the feasibility of the sequence $\{(\w_t, \sigma_t)\}_{t \geq 1}$ generated by Algorithm~\ref{Algorithm:PRG-IE}. Note that $\zero_m \in \WCal$ and $\zero_{n \times m} \in \XCal$ are not restrictive since $\WCal$ and $\XCal$ are commonly taken to be the ball of a norm.
  
This approach combines Malitsky's projected reflected gradient algorithm~\citep{Malitsky-2015-Projected} with Halpern-type inertial extrapolation~\citep{Halpern-1967-Fixed}. Such an integration has the following advantages: (i) it only performs one projection at each iteration; (ii) it achieves the strong convergence to one Bayesian equilibrium thanks to the extrapolation step. At each iteration, the projections $P_\WCal(\cdot)$ and $P_\Sigma(\cdot)$ are required and these operations are based on the inner product $\langle \cdot, \cdot\rangle_\HCal: \HCal \times \HCal \rightarrow \br$, which necessities the expectation with respect to a Bayesian prior $q$. In applications, variational inference or Markov chain Monte Carlo can be used to approximate this expectation.
\begin{theorem}\label{Theorem:PRGIE}
Under Assumption~\ref{Assumption:main} and~\ref{Assumption:PRG-IE}, the sequence $\{(\w_t, \sigma_t)\}_{t \geq 0} \in \WCal \times \Sigma$ generated by Algorithm~\ref{Algorithm:PRG-IE} satisfies  $\|\w_t-\w_\star\|^2+\EE_q[\|\sigma_t(\cb_d)-\sigma_\star(\cb_d)\|_F^2] \rightarrow 0$, where the point $(\w_\star, \sigma_\star) \in \WCal \times \Sigma$ is a unique Bayesian equilibrium.
\end{theorem}
We make some comments in the sequel. First, the sequence $\delta_t = 1/t$ in Algorithm~\ref{Algorithm:PRG-IE} is one specific choice and more general choices can be considered, as long as they satisfy $\delta_t \rightarrow 0$ and $\sum_{t=1}^{+\infty} \delta_t = +\infty$. The formula for updating the sequence $\{(\w_t, \sigma_t)\}_{t \geq 1}$ can also be generalized as follows:
\begin{equation*}
\begin{pmatrix} \w_{t+1} \\ \sigma_{t+1} \end{pmatrix} \ \leftarrow \ \delta_t \cdot g\begin{pmatrix} \w_t \\ \sigma_t \end{pmatrix} + \left(1-\delta_t\right)\begin{pmatrix} \widetilde{\w}_{t+1} \\ \widetilde{\sigma}_{t+1} \end{pmatrix},  
\end{equation*}
where $g: \HCal \mapsto \HCal$ is a contraction with the parameter $\kappa \in (0, 1)$. We set $g(\x) = \x/2$ for simplicity. Second, the strict monotonicity can be relaxed to monotonicity and it can be shown that the sequence $\{(\w_t, \sigma_t)\}_{t \geq 1}$ strongly converges to a particular Bayesian equilibrium. 

\section{Projected Gradient with Randomized Block Coordinate}\label{sec:alg_rand}
We present a \textit{projected gradient with randomized block coordinate} (PG-RBC) algorithm for solving Eq.~\eqref{prob:VI-finite}. To ease the analysis, we make the following assumption throughout this section before providing more intuitions for them.
\begin{algorithm}[!t]
\caption{Projected Gradient with Randomized Block Coordinate (PG-RBC)}\label{Algorithm:PG-RBC}
\begin{algorithmic}[1]
\STATE \textbf{Input:} strongly monotone parameter $\lambda > 0$; finite Bayesian prior $\{(p_k, \sv_k\}_{k=1}^K$; weight $\cb_l \in \br^n$.  
\STATE \textbf{Initialize:} $\w_0 \in \WCal$, $\sigma_0^k \in \XCal$ for all $k \in [K]$ and $\gamma_0 > \frac{1}{2\lambda}$.  
\FOR{$t=0,1,2,\ldots,T-1$}
\STATE Randomly pick up an index $j_t \in [K]$ according to $\PP(j_t=k)=p_k$ for all $k \in [K]$. 
\STATE Compute $(\w_{t+1}, \sigma_{t+1}^1, \ldots, \sigma_{t+1}^K)$ by 
\begin{align*}
\w_{t+1} & \leftarrow P_\WCal(\w_t - \gamma_t\nabla_\w\widehat{\theta}_l(\w_t, \sigma_t^{j_t}, \cb_l)), \\
\sigma_{t+1}^k & \leftarrow \left\{\begin{array}{ll}
P_\XCal(\sigma_t^k - \gamma_t\nabla_{\bar{X}}\widehat{\theta}_d(\w_t, \sigma_t^k, \sv_k)), & \textnormal{if } k = j_t, \\
\sigma_t^k, & \textnormal{otherwise.}
\end{array}\right. 
\end{align*}
\STATE Compute $\gamma_{t+1} \leftarrow \gamma/(t+1)$. 
\ENDFOR
\end{algorithmic}
\end{algorithm}
\begin{assumption}\label{Assumption:PG-RBC}
The Bayesian regression game $G = (\WCal, \Sigma, \widehat{\theta}_l, \widehat{\theta}_d, \cb_l, q)$ satisfies: (i) $\WCal$ and $\XCal$ are both compact such that there exists a positive constant $G$ such that $\|\nabla_\w\widehat{\theta}_l(\w, \bar{X}, \cb_l)\| \leq G$ for each $(\w, \bar{X}) \in \WCal \times \XCal$ and each $\cb_l \in \br^n$ fixed, and $\|\nabla_{\bar{X}}\widehat{\theta}_d(\w, \bar{X}, \sv)\|_F \leq G$ for each $(\w, \bar{X}) \in \WCal \times \XCal$ and each $\sv \in \{\sv_1, \ldots, \sv_k\}$; (ii) Given a fixed $\cb_l$, there exists a positive constant $\lambda$ such that\footnote{The constant $\lambda$ can depend on $\cb_l$ but is independent of the choice of $(\w, \bar{X})$ and $(\w', \bar{X}')$.} 
\begin{eqnarray*}
& & \langle \nabla_\w\widehat{\theta}_l(\w, \bar{X}, \cb_l) - \nabla_\w\widehat{\theta}_l(\w', \bar{X}', \cb_l), \w - \w'\rangle + \sum_{k=1}^K p_k\langle \nabla_{\bar{X}}\widehat{\theta}_d(\w, \bar{X}, \sv_k) - \nabla_{\bar{X}}\widehat{\theta}_d(\w', \bar{X}', \sv_k), \bar{X} - \bar{X}'\rangle \\ 
& &\geq  \lambda(\|\w - \w'\|^2 + \|\bar{X} - \bar{X}'\|_F^2), \qquad \qquad \forall (\w, \bar{X}), (\w', \bar{X}') \in \WCal \times \XCal.  
\end{eqnarray*}
\end{assumption} 
In Assumption~\ref{Assumption:PG-RBC}, the first condition can be interpreted as the strong monotonicity of $T$ defined by Eq.~\eqref{Eq:mapping-infinite-main} when the Bayesian prior $q$ is finite. The second condition naturally holds true if $\nabla_\w\widehat{\theta}_l(\cdot, \cdot, \cb_l)$ and $\nabla_{\bar{X}}\widehat{\theta}_d(\cdot, \cdot, \sv)$ are continuous for each $\cb_l \in \br^n$ fixed and each $\sv \in \{\sv_1, \ldots, \sv_k\}$.

The proposed approach combines projected gradient algorithm with randomized block coordinate update~\citep{Nesterov-2012-Efficiency, Wright-2015-Coordinate}. This design is more efficient than deterministic VI algorithms since the per iteration cost is $O(nm)$ which does not depend on $K$. Thus, our approach is favorable in application problems when the parameter $K$ is large.
\begin{theorem}\label{Theorem:PGRBC}
Under Assumption~\ref{Assumption:main} and~\ref{Assumption:PG-RBC}, the iterates $\{(\w_t, \sigma_t^1, \ldots, \sigma_t^k)\}_{t \geq 0}$ generated by Algorithm~\ref{Algorithm:PG-RBC} satisfy $\EE[\|\w_t-\w_\star\|^2+\sum_{k=1}^K \|\sigma_t^k-\sigma_\star^k\|_F^2] = O(1/t)$ where $(\w_\star, \sigma_\star^1, \ldots, \sigma_\star^K)$ is a unique Bayesian equilibrium.
\end{theorem}
We make some further comments. Since we do not assume any smoothness condition in Assumption~\ref{Assumption:PG-RBC}, the iteration complexity of $O(1/t)$ is the best possible we can hope for all the deterministic and stochastic algorithms in general\footnote{Convex optimization problems are a special class of the monotone VI problems. Thus, the problem complexity of (strongly) convex optimization implies that of (strongly) monotone VIs.}; see~\citet{Nemirovsky-1983-Complexity} for the reference. To this end, Theorem~\ref{Theorem:PGRBC} demonstrates that the complexity bound of Algorithm~\ref{Algorithm:PG-RBC} is tight in terms of the iteration number.

\section{Experiments}\label{sec:experiment}
We consider a spam email classification with quadratic cost functions on a real dataset, where the fixed-point approximation approach (denoted as Bayes-FP) proposed in \citet[Algorithm~1]{Grosshans-2013-Bayesian} can be implemented. We provide numerical evidences which demonstrate the advantage of the proposed stochastic optimization approach (denoted as Bayes-ADAM) in Eq.~\eqref{prob:opt-stochastic} over Bayes-FP. We also compare with two other baseline approaches, including a standard Ridge regression and a Nash equilibrium strategy. The Nash equilibrium strategy is simply the special case when the Bayesian prior is taken to be a point mass at its mean. Since Algorithm~\ref{Algorithm:PRG-IE} and~\ref{Algorithm:PG-RBC} are either more general or specialized to other settings, we believe it is unfair to compare them with Bayes-ADAM and Bayes-FP which can only be applied when the adversarial loss is quadratic. Thus, we exclude them here and leave further experimental investigations to future work.
\begin{figure*}[!t]
\centering
\begin{tabular}{ccc} 
\includegraphics[width=0.3\textwidth]{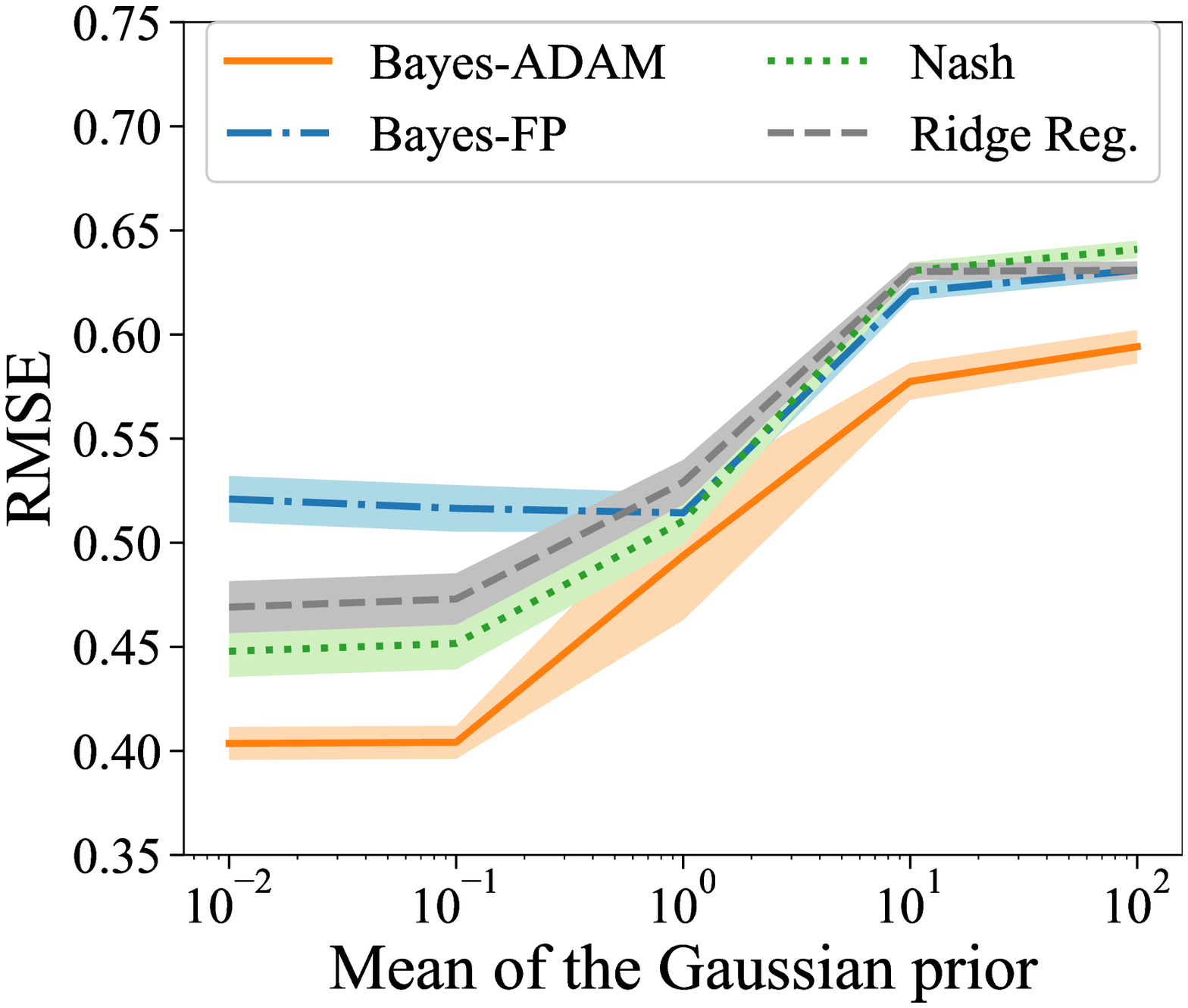} & 
\includegraphics[width=0.3\textwidth]{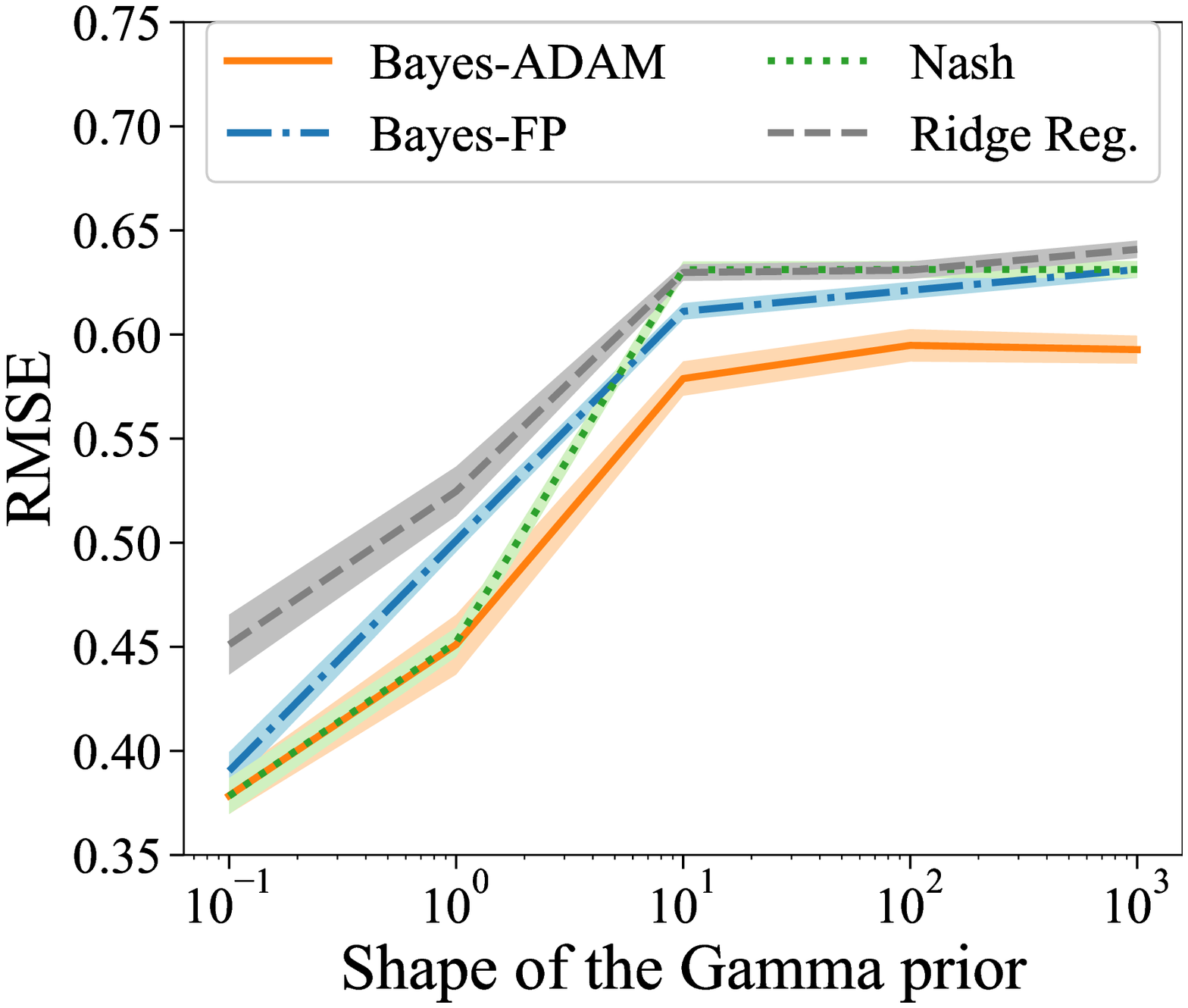} & 
\includegraphics[width=0.3\textwidth]{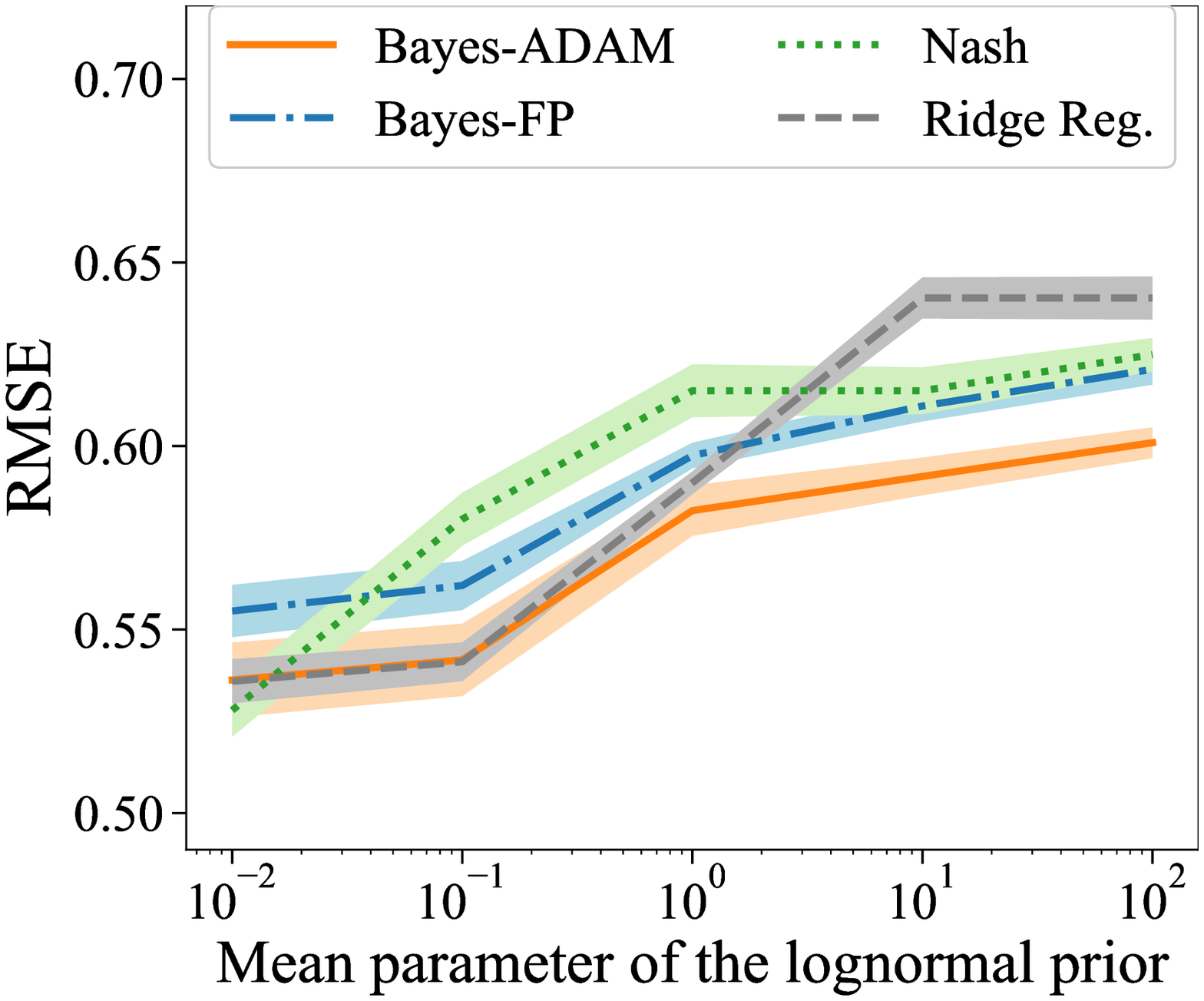} \\
\includegraphics[width=0.3\textwidth]{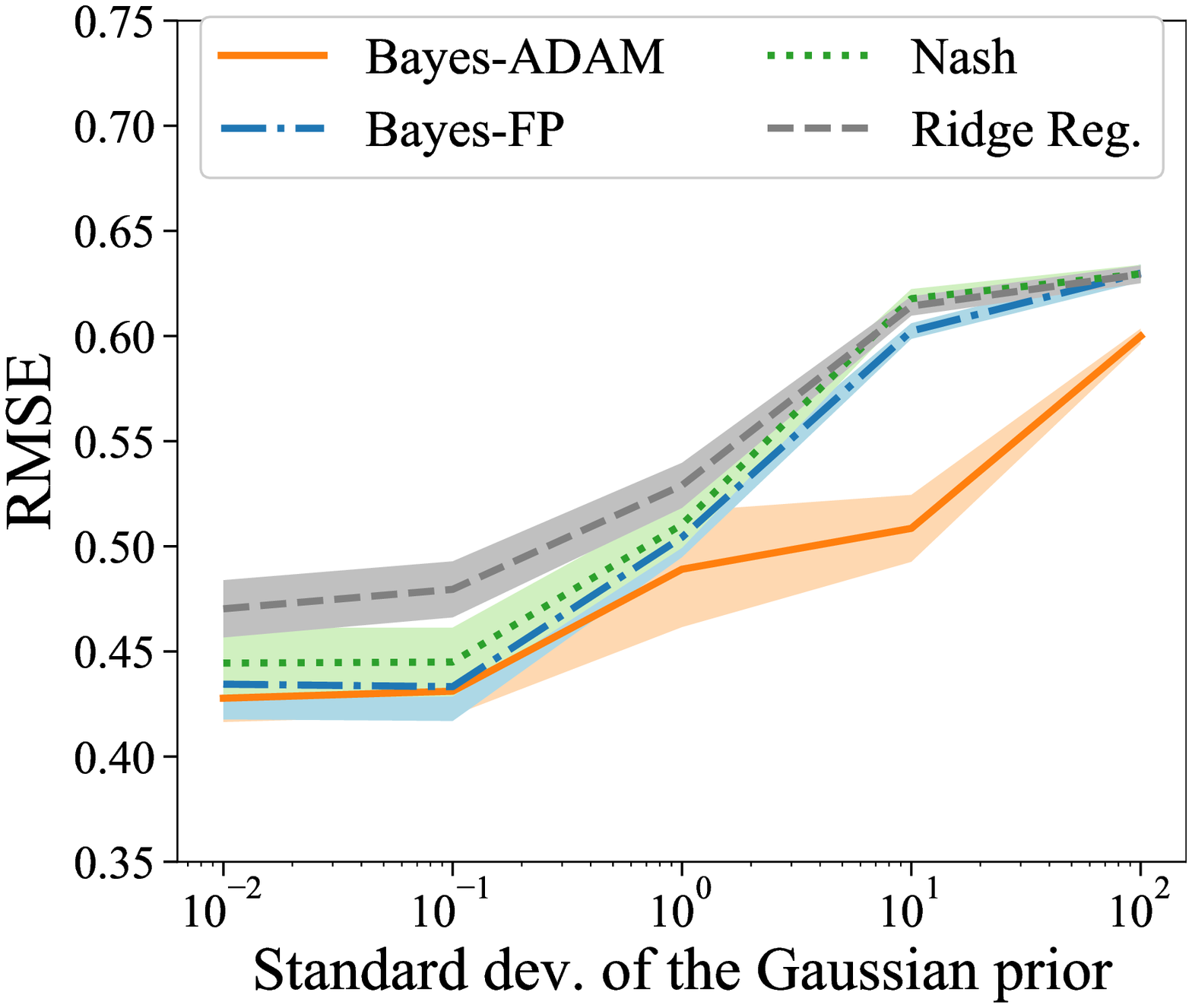} & 
\includegraphics[width=0.3\textwidth]{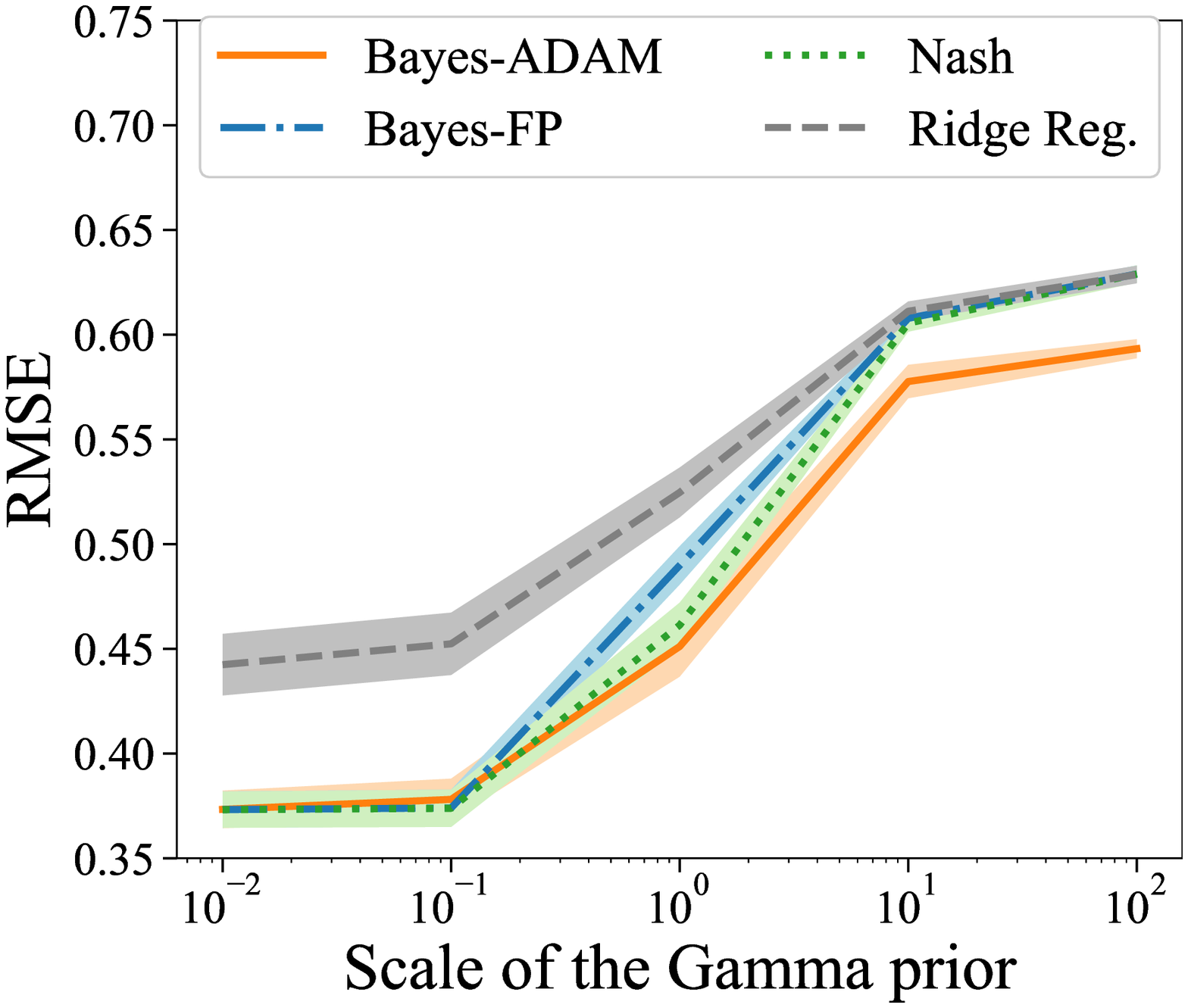} & 
\includegraphics[width=0.3\textwidth]{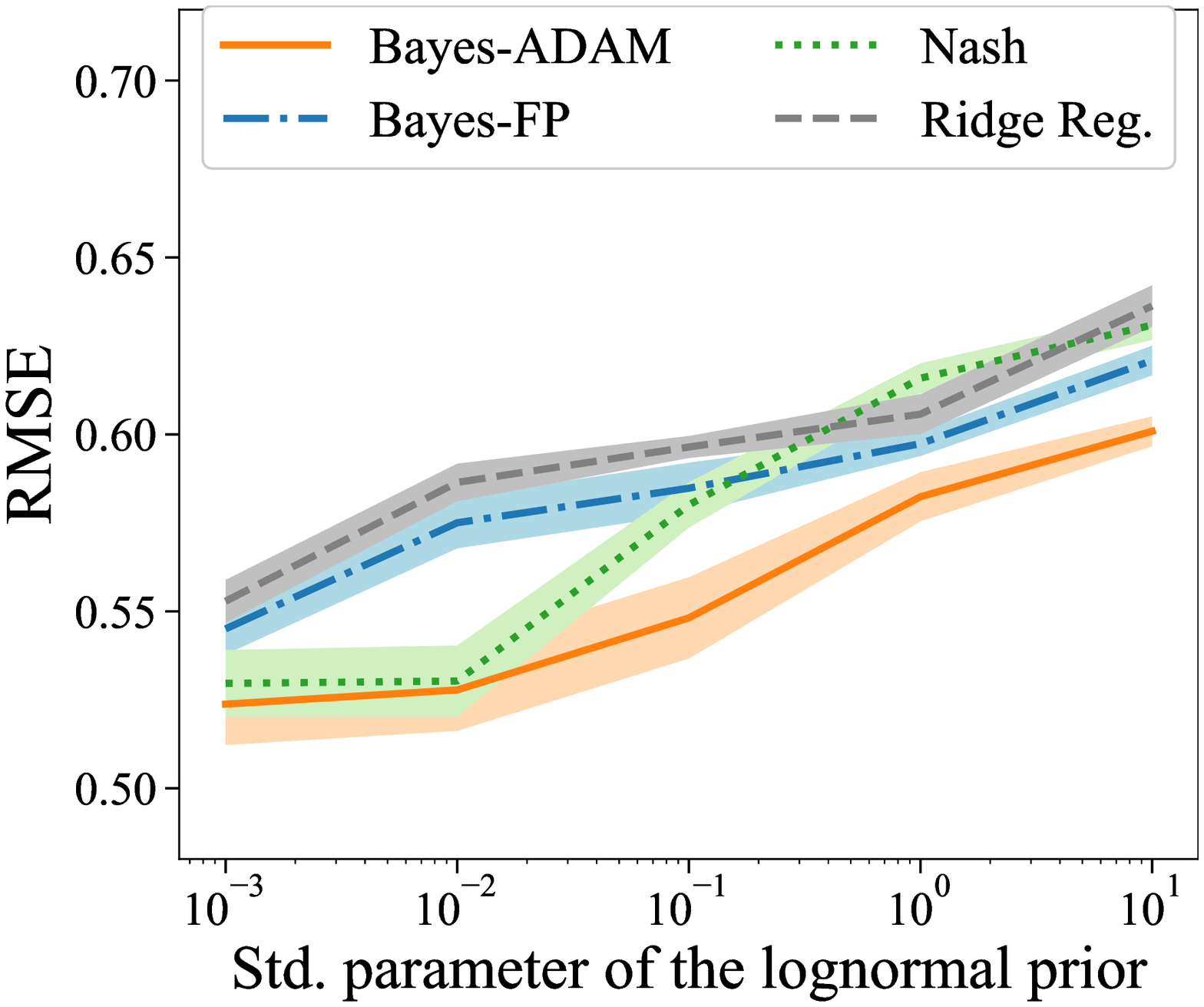}
\end{tabular}
\caption{RMSE comparisons for all algorithms with a Gaussian Bayesian prior and varying means and variances (left), or a Gamma Bayesian prior and varying shape and scale (middle), or a lognormal Bayesian prior and varying mean and normal variance (right). The $x$-axis is on log scale with base=10. The RMSE are computed over 10 random train/test splits.}\label{fig:results}
\end{figure*}

\paragraph{Dataset.} We use the \textit{Spambase} dataset~\citep{Dua:2019}, which contains 4601 examples. The prediction task is to identify if an email is spam, and the binary label for each sample denotes whether it was considered spam (1) or not (0). Most of the attributes indicate whether a particular word or character was frequently occuring in the e-mail. The run-length attributes measure the length of sequences of consecutive capital letters. We refer the interested reader to \textit{http://archive.ics.uci.edu/ml/datasets/Spambase/} for more details. 

\paragraph{Experimental setup.} We study how Bayes-ADAM and Bayes-FP perform against an adversary that chooses a strategy according to a Bayesian equilibrium for different parameters of the Bayesian prior. We also compare these two algorithms to two baselines, including a standard Ridge regression, and a strategy at a Nash equilibrium when the prior is set to be a point mass.  For the Bayesian regression setting, we adopt a similar setup as in \citet{Grosshans-2013-Bayesian}. In each repetition, we construct a pair of disjoint train and test sets drawn from the whole dataset at random. Both the train and test sets contain 500 datapoints. We then compute two Bayesian equilibrium points on each set. We extract the learner’s model from the trainset's equilibrium point, and transform the data points from the testset's equilibrium point after drawing actual costs from the prior. We then test the model on the transformed test data. We draw 500 random samples of $c_d$ in testing. We use root mean squared error (RMSE) to evaluate the predictions, computed using scikit-learn~\citep{scikit-learn}. We set the loss functions of the learner and adversary to be $f_l(\w, \bar \x, y)=(\bar \x^\top\w-y)^2$, $\Omega_l(\w)=\|\w\|_2^2$ and $f_d(\w, \x, z)=(\x^\top\w-z)^2$, $\Omega_d(X, \bar{X})=\|X-\bar{X}\|_F^2$ with $c_{\ell, i} = 0.1$. 

For the implementation of \citet[Algorithm~1]{Grosshans-2013-Bayesian}, we perform the fixed point update for 20 iterations and use a first order Taylor expansion to approximate the adversary's best response. We also use this same procedure to compute a Nash equilibrium, where the Bayesian prior is set to be a point mass at its mean. We then compare it to our method (Bayes-ADAM), which solves a stochastic nonconvex optimization problem. Specifically, we solve for Eq~\eqref{prob:opt-stochastic} using random samples of $c_d$. During the training for the stochastic optimization, we draw a batch of random samples in each round and compute the gradients using these samples. All gradient steps were implemented using PyTorch’s Adam optimizer\footnote{https://pytorch.org/docs/stable/optim.html}. The total number of random $c_d$ samples used is 1000. We run the algorithm Bayes-ADAM for 20 epochs, where the learning rate is tuned within $\{0.001, 0.01, 0.1\}$, and the batch size is tuned over $\{32, 64, 128\}$. We implemented the ridge regression algorithm with the scikit-learn package~\citep{scikit-learn}, where the regularization hyperparameter is tuned over $\{0.01, 0.1, 1\}$. Results for all algorithms are averaged over ten random train/test splits. 

\paragraph{Experimental results.} We compare Algorithm~1 in \citet{Grosshans-2013-Bayesian} (denoted as Bayes-FP) with our proposed stochastic optimization procedure (denoted as Bayes-ADAM), in the settings of three continuous or discrete Bayesian prior types, including the multivariate Gaussian distribution, Gamma distribution, and lognormal distribution. For each type of prior distribution, we vary the mean and standard deviation for the Gaussian prior; scale and shape for the Gamma prior; mean and standard deviation parameter for the lognormal prior's corresponding normal distribution. When not varied, the corresponding parameter is set to 1. Figure~\ref{fig:results} presents the performance of four algorithms for the Bayesian regression games with three different types of priors. Bayes-ADAM outperforms the Bayes-FP algorithm as well as the other two baselines including the Ridge regression and the Nash strategy. In particular, Bayes-ADAM is able to achieve lower RMSE when the Gaussian prior or the lognormal prior has a larger mean or variance, and when the Gamma prior has a larger shape or scale. On the other hand, the Nash strategy and ridge regression can achieve relatively low RMSE when the Bayesian prior has a lower variance, but fails to achieve low RMSE when the variance becomes large.

\section{Conclusions}\label{sec:conclu}
We have presented a computational theory of Bayesian regression games, making links to general Bayesian games and variational inequalities while focusing on an algorithm viewpoint. We provide sufficient conditions for the existence and uniqueness of equilibria by using an infinite-dimensional VI model, generalizing~\citet[Theorem~2]{Grosshans-2013-Bayesian}. We also discuss two special cases in which the infinite-dimensional VI reduces to a high-dimensional VI or a stochastic optimization in Euclidean space. We propose the algorithms for the computation of equilibria and provide numerical results to demonstrate the effectiveness of our framework in a classification setting.

\section*{Acknowledgements}
This work was supported in part by the Mathematical Data Science program of the Office of Naval Research under grant number N00014-18-1-2764.

\bibliographystyle{plainnat}
\bibliography{ref}

\newpage\appendix
\section{Further Background Material}
In this section, we provide the basic ideas and some additional background materials for the development of our PRG-IE and PG-RBC algorithms. Some discussions on the relevant algorithms are also included. 

\paragraph{PRG-IE:} We start with a brief overview of the projected reflected gradient algorithm for solving the variational inequality (VI) in Hilbert space. Let $\SCal$ be a nonempty, closed and convex set of a Hilbert space $\HCal$ with the inner product $\langle \cdot, \cdot\rangle_\HCal$, and $T: \HCal \rightarrow \HCal$ be \textit{strictly monotone} and $\ell$-smooth for some constant $\ell >0$: for $\forall\x, \x' \in \HCal$, $\|T(\x) - T(\x')\|_\HCal \leq \ell\|\x-\x'\|_\HCal$. Then, we consider the problem of finding a point $\x_\star \in \SCal$ such that 
\begin{equation}\label{prob:VI-general}
\langle \x - \x_\star, T(\x_\star)\rangle_\HCal \geq 0, \quad \textnormal{for all } \x \in \SCal. 
\end{equation}  
A projected reflected gradient algorithm $\x_{t+1} \leftarrow P_\SCal(\x_t - \gamma \cdot T(2\x_t - \x_{t-1}))$ can be applied for solving this problem where the stepsize $\lambda \in (0, (\sqrt{2}-1)/\ell)$, and $P_\SCal(\cdot)$ is the orthogonal projection onto a closed set $\SCal$. From the update formula, we see that this algorithm has a very simple and elegant structure, which only requires evaluating $T(\cdot)$ and $P_\SCal(\cdot)$ once at each iteration. Thus, it is more computationally appealing than the Korpelevich's extragradient algorithm~\cite{Korpelevich-1976-Extragradient}, Popov's modified Arrow-Hurwicz algorithm~\cite{Popov-1980-Modification}, Tseng's forward-backward splitting algorithm~\cite{Tseng-2000-Modified} and some other algorithms~\cite{Solodov-1999-New, Malitsky-2014-Extragradient}. 

Note that the VI in Eq.~\eqref{prob:VI-infinite} is in the form of Eq.~\eqref{prob:VI-general} with a Hilbert space $\HCal$ consisting of (an equivalence class of) a function $\beta: \br^n \mapsto \br^m \times \br^{n \times m}$ with the inner product $\langle \cdot, \cdot\rangle_\HCal: \HCal \times \HCal \rightarrow \br$ defined by Eq.~\eqref{Def:inner-product} and the mapping $T$ defined by Eq.~\eqref{Eq:mapping-infinite}. Under Assumption~\ref{Assumption:main}, we assume that $T$ is strictly monotone and $\WCal \times \Sigma$ is compact. Then,~\cite[Theorem~3.2]{Malitsky-2015-Projected} guarantees that the sequence generated by projected reflected gradient algorithm weakly converges to a unique Bayesian equilibrium. However, in the infinite-dimensional setting, strong convergence (or norm convergence) is often much more desirable than weak convergence, since it guarantees that the physically tangible property, the error $\|\x_t - \x_\star\|_\HCal^2$ eventually become arbitrarily small~\cite{Bauschke-2001-Weak}. The importance of strong convergence is also demonstrated by~\cite{Guler-1991-Convergence} for convex optimization that the convergence rate of the sequence of objectives $\{f(\x_t)\}_{t \geq 0}$ is better when $\{\x_t\}_{t \geq 0}$ with strong convergence than weak convergence. This encourages the strong convergence theorems for various algorithms in Hilbert space~\cite{Solodov-2000-Forcing, Nadezhkina-2006-Strong}.

\paragraph{PG-RBC:} The variational inequality (VI) in Eq.~\eqref{prob:VI-finite} is the problem of finding a point $(\w_\star, \sigma_\star^1, \ldots, \sigma_\star^K) \in \SCal = \WCal \times \XCal \times \ldots \times \XCal$ such that 
\begin{equation*}
\sum_{k=1}^K p_k\left[\left\langle \w-\w_\star, \nabla_\w\widehat{\theta}_l(\w_\star, \sigma_\star^k, \cb_l)\right\rangle + \left\langle \sigma^k - \sigma_\star^k, \nabla_{\bar{X}}\widehat{\theta}_d(\w_\star, \sigma_\star^k, \sv_k)\right\rangle\right] \ \geq \ 0,
\end{equation*}
for all $\w \in \WCal$ and $\sigma^k \in \XCal$ for all $k \in [K]$. Then, by defining the variable $\x \in \br^m \times \br^{n \times m} \times \ldots \times \br^{n \times m}$ and a mapping $T: \br^m \times \br^{n \times m} \times \ldots \times \br^{n \times m} \mapsto \br^m \times \br^{n \times m} \times \ldots \times \br^{n \times m}$ as follows, 
\begin{equation*}
T\begin{pmatrix} \w \\ \sigma^1 \\ \vdots \\ \sigma^K \end{pmatrix} \ = \ \begin{pmatrix} \sum_{k=1}^K \nabla_\w\widehat{\theta}_l(\w, \sigma^k, \cb_l) \\ \nabla_{\bar{X}}\widehat{\theta}_d(\w, \sigma^1, \sv_1) \\ \vdots \\ \nabla_{\bar{X}}\widehat{\theta}_d(\w, \sigma^K, \sv_K),   \end{pmatrix}
\end{equation*}
we can reformulate the above problem in the compact form as follows, 
\begin{equation*}
\langle \x - \x_\star, T(\x_\star)\rangle_\HCal \geq 0, \quad \textnormal{for all } \x \in \SCal = \WCal \times \XCal \times \ldots \times \XCal. 
\end{equation*} 
A projected gradient algorithm $\x_{t+1} \leftarrow P_\SCal(\x_t - \gamma \cdot T(\x_t))$ can be applied but becomes problematic when the problem dimension $m$, the number of data samples $n$ and the range of a Bayesian prior $K$ are huge. Indeed, the algorithm require performing arithmetic operations of order $nmK$ per iteration and the projection step is another source of inefficiency for huge-size problem. The coordinate update and more generally block coordinate update, which are commonly used to address this issue and improve the computational efficiency, are rooted in the optimization community~\cite{Bertsekas-1989-Parallel}. During the past decade, the \textit{randomized coordinate update} has emerged as one of the most popular coordinate update schemes and were extensively studied~\cite{Nesterov-2012-Efficiency, Richtarik-2014-Iteration, Fercoq-2015-Accelerated}.

\section{Postponed Proofs in Section~\ref{sec:prelim}}
This section lays out the detailed proofs for Lemma~\ref{lemma:VI}, Theorem~\ref{Theorem:VI},~\ref{Theorem:VI-existence} and~\ref{Theorem:existence-uniqueness}, and Corollary~\ref{corollary:VI-existence-first},~\ref{corollary:VI-existence-second},~\ref{corollary:VI-existence-third} and~\ref{corollary:VI-existence-fourth}. 

\paragraph{Proof of Lemma~\ref{lemma:VI}.} Note that $\widehat{\theta}_l(\cdot, \bar{X}, \cb_l): \WCal \rightarrow \br$ is convex and continuously differentiable for each $\bar{X} \in \XCal$. By the Lebesgue monotone convergence theorem, we have
\begin{equation}\label{inequality:VI-first}
\nabla_\w\EE_q[\widehat{\theta}_l(\w_\star, \sigma_\star(\cb_d), \cb_l)] = \EE_q[\nabla_\w\widehat{\theta}_l(\w_\star, \sigma_\star(\cb_d), \cb_l)]. 
\end{equation}
For $\w \in \WCal$, we let $g(t)=\EE_q[\widehat{\theta}_l(\w_\star+t(\w-\w_\star), \sigma_\star(\cb_d), \cb_l)]$. Since $(\w_\star, \sigma_\star) \in \WCal \times \Sigma$ is a Bayesian equilibrium, we have $g(t) \geq g(0)$ for all $t \in \br$. This implies that $g'(0) \geq 0$. By definition,
\begin{eqnarray*}
g'(0) & = & \left\langle \w-\w_\star, \left\{\nabla_\w\EE_q[\widehat{\theta}_l(\w_\star+t(\w-\w_\star), \sigma_\star(\cb_d), \cb_l)]\vert_{t=0}\right\}\right\rangle \\
&= &\langle \w-\w_\star, \nabla_\w\EE_q[\widehat{\theta}_l(\w_\star, \sigma_\star(\cb_d), \cb_l)\rangle \\
& \overset{\textnormal{Eq.}~\eqref{inequality:VI-first}}{=} & \langle \w-\w_\star, \EE_q[\nabla_\w\widehat{\theta}_l(\w_\star, \sigma_\star(\cb_d), \cb_l)]\rangle \ = \ \EE_q[\langle \w-\w_\star, \nabla_\w\widehat{\theta}_l(\w_\star, \sigma_\star(\cb_d), \cb_l)\rangle]. 
\end{eqnarray*}
Putting these pieces together yields the first inequality in Eq.~\eqref{lemma:VI-main}. In addition, notice that the cost function $\widehat{\theta}_d(\w, \cdot, \cb_d): \XCal \rightarrow \br$ is convex and continuously differentiable for each $\w \in \WCal$ and no expectation is involved now: By the similar argument, we obtain the second inequality in Eq.~\eqref{lemma:VI-main}. 

Conversely, we show that Eq.~\eqref{lemma:VI-main} guarantees that the strategy profile $(\w_\star, \sigma_\star) \in \WCal \times \Sigma$ is a Bayesian equilibrium. Note that the function $g(t)=\EE_q[\widehat{\theta}_l(\w_\star+t(\w-\w_\star), \sigma_\star(\cb_d), \cb_l)]$ is convex since $\widehat{\theta}_l(\cdot, \bar{X}, \cb_l): \WCal \rightarrow \br$ is convex for each $\bar{X} \in \XCal$. Thus, $g(t) \geq g(0)+tg'(0)$ for each $t \in \br$. By definition, we have
\begin{equation*}
\EE_q[\widehat{\theta}_l(\w, \sigma_\star(\cb_d), \cb_l)] = g(1) \geq g(0) + tg'(0) = \EE_q[\widehat{\theta}_l(\w_\star, \sigma_\star(\cb_d), \cb_l)] + \EE_q[\langle \w-\w_\star, \nabla_\w\widehat{\theta}_l(\w_\star, \sigma_\star(\cb_d), \cb_l)\rangle]. 
\end{equation*}
Combining the above inequality with the first inequality in Eq.~\eqref{lemma:VI-main}, we have
\begin{equation*}
\EE_q[\widehat{\theta}_l(\w, \sigma_\star(\cb_d), \cb_l)] \ \geq \ \EE_q[\widehat{\theta}_l(\w_\star, \sigma_\star(\cb_d), \cb_l)], \quad \textnormal{for all } \w \in \WCal. 
\end{equation*}
Using a similar argument and the second inequality in Eq.~\eqref{lemma:VI-main}, we have 
\begin{equation*}
\widehat{\theta}_d(\w_\star, \bar{X}, \cb_d) \ \geq \ \widehat{\theta}_d(\w_\star, \sigma_\star(\cb_d), \cb_d), \quad \textnormal{for all } \bar{X} \in \XCal.  
\end{equation*}
This completes the proof. 

\paragraph{Proof of Theorem~\ref{Theorem:VI}.} We first show the ``only if" direction. Indeed, let $(\w_\star, \sigma_\star) \in \WCal \times \Sigma$ be a Bayesian equilibrium for the Bayesian regression game $G = (\WCal, \Sigma, \widehat{\theta}_l, \widehat{\theta}_d, \cb_l, q)$, we derive from Lemma~\ref{lemma:VI} that Eq.~\eqref{lemma:VI-main} holds true. This implies that, for all $(\w, \sigma) \in \WCal \times \Sigma$ and for a.e. $\omega \in \Omega$, 
\begin{eqnarray*}
\EE_q[\langle\w-\w_\star, \nabla_\w\widehat{\theta}_l(\w_\star, \sigma_\star(\cb_d), \cb_l)\rangle] & \geq & 0, \\
\langle \sigma(\cb_d) - \sigma_\star(\cb_d), \nabla_{\bar{X}}\widehat{\theta}_d(\w_\star, \sigma_\star(\cb_d), \cb_d)\rangle & \geq & 0.  
\end{eqnarray*}
Summing up the above two inequalities and taking the expectation over the distribution $q$ yields the desired inequality in Eq.~\eqref{prob:VI-infinite}. Then it suffices to show the ``if" direction. Specifically, we show that if $(\w_\star, \sigma_\star) \in \WCal \times \Sigma$ is not a Bayesian equilibrium for the Bayesian regression game $G = (\WCal, \Sigma, \widehat{\theta}_l, \widehat{\theta}_d, \cb_l, q)$, then Eq.~\eqref{lemma:VI-main} does not hold true. By Lemma~\ref{lemma:VI}, if $(\w_\star, \sigma_\star) \in \WCal \times \Sigma$ is not a Bayesian equilibrium, we have 
\begin{equation}\label{inequality:VI-second}
\EE_q[\langle\w-\w_\star, \nabla_\w\widehat{\theta}_l(\w_\star, \sigma_\star(\cb_d), \cb_l)\rangle] < 0 \quad \textnormal{for some } \w \in \WCal, 
\end{equation}
or there exists $E \subseteq \Omega$ with $\PP(E)>0$ such that, for each $\omega \in E$,  
\begin{equation}\label{inequality:VI-third}
\langle\bar{X} - \sigma_\star(\cb_d), \nabla_{\bar{X}}\widehat{\theta}_d(\w_\star, \sigma_\star(\cb_d), \cb_d)\rangle < 0 \quad \textnormal{for some } \bar{X} \in \XCal.
\end{equation}
Let $(\w', \sigma') \in \WCal \times \Sigma$ be defined by 
\begin{eqnarray*}
\w' & = & \left\{\begin{array}{ll}
\w & \textnormal{if Eq.~\eqref{inequality:VI-second} holds true}, \\
\w_\star  & \textnormal{otherwise}.
\end{array}\right. \\
\sigma'(\cb_d(\omega)) & = & \left\{\begin{array}{ll}
\bar{X} & \textnormal{if Eq.~\eqref{inequality:VI-third} holds true and } \omega \in E, \\
\sigma_\star(\cb_d(\omega))  & \textnormal{otherwise}.
\end{array}\right. 
\end{eqnarray*}
By simple calculations, we have
\begin{equation*}
\EE_q\left[\left\langle \w'-\w_\star, \nabla_\w\widehat{\theta}_l(\w_\star, \sigma_\star(\cb_d), \cb_l)\right\rangle + \left\langle \sigma'(\cb_d) - \sigma_\star(\cb_d), \nabla_{\bar{X}}\widehat{\theta}_d(\w_\star, \sigma_\star(\cb_d), \cb_d)\right\rangle\right] < 0. 
\end{equation*}
This completes the proof.  

\paragraph{Proof of Theorem~\ref{Theorem:VI-existence}.} We provide a key notion of monotonicity which is pivotal in the classical VI literature and summarize in Proposition~\ref{prop:BHS} the celebrated existence theorem for an infinite-dimensional VI.   
\begin{definition}[Monotonicity]
Let $\HCal$ be a Hilbert space with the inner product $\langle \cdot, \cdot\rangle_\HCal$, we define $\SCal \subseteq \HCal$ as a closed and convex set and $T:\SCal \rightarrow \HCal$ as a mapping. Then,  
\begin{enumerate}
\item $T$ is monotone if $\langle T\beta-T\beta', \beta-\beta'\rangle_\HCal \geq 0$ for each $\beta, \beta' \in \SCal$. 
\item $T$ is strictly monotone if $\langle T\beta-T\beta', \beta-\beta'\rangle_\HCal > 0$ for each $\beta, \beta' \in \SCal$ with $\beta \neq \beta'$.  
\item $T$ is $\lambda$-strongly monotone ($\lambda > 0$) if $\langle T\beta-T\beta', \beta-\beta'\rangle_\HCal > \lambda\|\beta-\beta'\|_\HCal^2$ for each $\beta, \beta' \in \SCal$.   
\end{enumerate}
\end{definition} 
\begin{proposition}[Browder-Hartman-Stampacchia]\label{prop:BHS}
Let $\HCal$ be a Hilbert space with the inner product $\langle \cdot, \cdot\rangle_\HCal$. Define $\SCal \subseteq \HCal$ as a nonempty, closed and convex set, and $T: \SCal \rightarrow \HCal$ as a monotone mapping. If the following conditions hold true:
\begin{enumerate}
\item The mapping $t \mapsto \langle T((1-t)\beta+t\beta'), \alpha\rangle_\HCal$ from $[0, 1]$ to $\br$ is continuous for all $\beta, \beta' \in \SCal$ and $\alpha \in \HCal$. 
\item The set $\SCal$ is compact, or there exists $\beta_0 \in \SCal$ such that $\frac{\langle T\beta, \beta-\beta_0\rangle_\HCal}{\|\beta\|_\HCal}\rightarrow +\infty$ as $\|\beta\|_\HCal \rightarrow +\infty$. 
\end{enumerate}
Then, there exists $\bar{\beta} \in \SCal$ such that $\langle T\bar{\beta}, \beta-\bar{\beta}\rangle_\HCal \geq 0$ for all $\beta \in \SCal$. 
\end{proposition}
Let us consider the Bayesian regression game $G = (\WCal, \Sigma, \widehat{\theta}_l, \widehat{\theta}_d, \cb_l, q)$ and its Bayesian equilibrium in terms of Eq.~\eqref{prob:VI-infinite}. Then, we can define a Hilbert space $\HCal$ consisting of (an equivalence class of) a function $\beta: \br^n \mapsto \br^m \times \br^{n \times m}$ with the inner product $\langle \cdot, \cdot\rangle_\HCal: \HCal \times \HCal \rightarrow \br$ by
\begin{equation}\label{Def:inner-product}
\left\langle\begin{pmatrix} \w \\ \sigma\end{pmatrix}, \begin{pmatrix} \w' \\ \sigma' \end{pmatrix}\right\rangle_\HCal \ = \ \EE_q[\langle\w(\cb_d), \w'(\cb_d)\rangle + \langle \sigma(\cb_d), \sigma'(\cb_d)\rangle].    
\end{equation}
Note that each element in $\WCal \subseteq \br^m$ can be regarded as a constant function from $\br^n$ to $\br^m$ whose value equal to this element. Then, we denote the set of these constant function by $\Sigma_\WCal$ (an equivalence class of $\WCal$) and define the mapping $T: \Sigma_\WCal \times \Sigma \rightarrow \HCal$ as follows, 
\begin{equation}\label{Eq:mapping-infinite}
T\begin{pmatrix} \w \\ \sigma(\cdot) \end{pmatrix} \ = \ \begin{pmatrix} \nabla_\w\widehat{\theta}_l(\w, \sigma(\cdot), \cb_l) \\ \nabla_{\bar{X}}\widehat{\theta}_d(\w, \sigma(\cdot), \cdot) \end{pmatrix} \ \in \ \HCal,   
\end{equation}
To this end, the computation of a Bayesian equilibrium is equivalent to solving a VI in the space $\HCal$. This allows us to analyze the existence and uniqueness of a Bayesian equilibrium under the Browder-Hartman-Stampacchia's VI framework (cf. Proposition~\ref{prop:BHS}). For example, the existence of a Bayesian equilibrium is guaranteed by the continuity and monotonicity of $T$ as well as some additional conditions on $\WCal \times \Sigma$.

To prove the existence of a solution, we show that Eq.~\eqref{prob:VI-infinite} is a special case of the Browder-Hartman-Stampacchia VIs. Indeed, we set a Hilbert space $\HCal$ consisting of (an equivalence class of) a function $\beta: \br^n \mapsto \br^m \times \br^{n \times m}$ with the inner product $\langle \cdot, \cdot\rangle_\HCal: \HCal \times \HCal \rightarrow \br$ defined by Eq.~\eqref{Def:inner-product}. By abuse of notation, any element $\w \in \WCal$ define a constant function $\w(\cdot) \in \Sigma_\WCal$ and vice versa. This implies that $\Sigma_\WCal$ is an equivalent class of $\WCal$. Thus, $\SCal = \Sigma_\WCal \times \Sigma$ is a nonempty, closed and convex subset of $\HCal$. In addition, a mapping $T: \SCal \mapsto \HCal$ defined by Eq.~\eqref{Eq:mapping-infinite} is continuous and monotone. Putting these pieces together with either the compactness of $\SCal$ or the ccertain ondition with Eq.~\eqref{Condition:VI-existence} yields that all the assumptions in Proposition~\ref{prop:BHS} hold true and Eq.~\eqref{prob:VI-infinite} is a special case of the Browder-Hartman-Stampacchia VIs. Therefore, we conclude from Proposition~\ref{prop:BHS} that the VI in Eq.~\eqref{prob:VI-infinite} has at least one solution. This completes the proof.

\paragraph{Proof of Corollary~\ref{corollary:VI-existence-first}.} Under Assumption~\ref{Assumption:main}, Theorem~\ref{Theorem:VI} implies that a Bayesian equilibrium must be a solution of the VI in Eq.~\eqref{prob:VI-infinite}, Thus, it suffices to verify the assumptions in Theorem~\ref{Theorem:VI-existence}. Indeed, Assumption~\ref{Assumption:main} guarantees that $T$ defined by Eq.~\eqref{Eq:mapping-infinite} is continuous. In addition, $T$ is monotone and there exists $(\w_0, \sigma_0) \in \WCal \times \Sigma$ such that, for all $(\w, \sigma) \in \WCal \times \Sigma$ satisfying $\|\w\|^2+\EE_q[\|\sigma(\cb_d)\|_F^2] \rightarrow +\infty$, Eq.~\eqref{Condition:VI-existence} holds true. Therefore, we conclude the desired result.    

\paragraph{Proof of Corollary~\ref{corollary:VI-existence-second}.} The proof is nearly the same as that of Corollary~\ref{corollary:VI-existence-first} and the only difference is that, we assume the compactness of the action space $\WCal \times \Sigma$, which is another condition of Theorem~\ref{Theorem:VI-existence}. Thus, by using Theorem~\ref{Theorem:VI-existence} again, we conclude the desired result. 

\paragraph{Proof of Theorem~\ref{Theorem:existence-uniqueness}.}Theorem~\ref{Theorem:VI-existence} guarantees the existence of at least one solution of the VI in Eq.~\eqref{prob:VI-infinite}. Thus, it suffices to show that at most one solution exists if $T$ defined by Eq.~\eqref{Eq:mapping-infinite} is strictly monotone. Using the proof by contradiction, suppose that the VI in Eq.~\eqref{prob:VI-infinite} has two different solutions. Given the inner product $\langle\cdot, \cdot\rangle_\HCal$ is defined by Eq.~\eqref{Def:inner-product}, we have  
\begin{equation*}
\left\langle\begin{pmatrix} \w_1-\w_2 \\ \sigma_1 - \sigma_2 \end{pmatrix}, T\begin{pmatrix} \w_1 \\ \sigma_1 \end{pmatrix}\right\rangle_\HCal \ \geq \ 0, \qquad \left\langle\begin{pmatrix} \w_2-\w_1 \\ \sigma_2 - \sigma_1 \end{pmatrix}, T\begin{pmatrix} \w_2 \\ \sigma_2 \end{pmatrix}\right\rangle_\HCal \ \geq \ 0. 
\end{equation*}
Summing up the above two inequalities yields:
\begin{equation*}
\left\langle \begin{pmatrix} \w_1 - \w_2 \\ \sigma_1 - \sigma_2 \end{pmatrix}, T\begin{pmatrix} \w_1 \\ \sigma_1 \end{pmatrix} - T\begin{pmatrix} \w_2 \\ \sigma_2 \end{pmatrix}\right\rangle_\HCal  \ \geq \ 0. 
\end{equation*} 
Since a mapping $T: \Sigma_\WCal \times \Sigma \rightarrow \br^m \times \Sigma_0$ is strictly monotone, we have $\w_1=\w_2$ and $\sigma_1=\sigma_2$ almost everywhere. This leads to a contradiction and completes the proof.

\paragraph{Proof of Corollary~\ref{corollary:VI-existence-third}.} Under Assumption~\ref{Assumption:main}, Theorem~\ref{Theorem:VI} implies that a Bayesian equilibrium must be a solution of the VI in Eq.~\eqref{prob:VI-infinite}, Thus, it suffices to verify the assumptions in Theorem~\ref{Theorem:existence-uniqueness}. Note that $T$ defined by Eq.~\eqref{Eq:mapping-infinite} is $\lambda$-strongly monotone and thus strictly monotone. In order to prove the uniqueness of a Bayesian equilibrium using Theorem~\ref{Theorem:existence-uniqueness}, it remains to show that there exists $(\w_0, \sigma_0) \in \WCal \times \Sigma$ such that, for all $(\w, \sigma) \in \WCal \times \Sigma$ satisfying $\|\w\|^2+\EE_q[\|\sigma(\cb_d)\|_F^2] \rightarrow +\infty$, Eq.~\eqref{Condition:VI-existence} holds true. Since $T$ is $\lambda$-strongly monotone, we have
\begin{eqnarray*}
& & \EE_q\left[\left\langle \w-\w_0, \nabla_\w\widehat{\theta}_l(\w, \sigma(\cb_d), \cb_l)\right\rangle + \left\langle \sigma(\cb_d) - \sigma_0(\cb_d), \nabla_{\bar{X}}\widehat{\theta}_d(\w, \sigma(\cb_d), \cb_d)\right\rangle\right] \\ 
& &\overset{\textnormal{Eq.}~\eqref{Def:inner-product}}{=}  \left\langle\begin{pmatrix} \w-\w_0 \\ \sigma - \sigma_0 \end{pmatrix}, T\begin{pmatrix} \w \\ \sigma \end{pmatrix}\right\rangle_\HCal \\
&&\ \geq \ \left\langle\begin{pmatrix} \w-\w_0 \\ \sigma - \sigma_0 \end{pmatrix}, T\begin{pmatrix} \w_0 \\ \sigma_0 \end{pmatrix}\right\rangle_\HCal + \lambda\left(\|\w-\w_0\|^2 + \EE_q[\|\sigma(\cb_d)-\sigma_0(\cb_d)\|_F^2]\right). 
\end{eqnarray*}
Let $(\w, \sigma) \in \WCal \times \Sigma$ satisfy $\|\w\|^2+\EE_q[\|\sigma(\cb_d)\|_F^2] \rightarrow +\infty$ and $(\w_0, \sigma_0) \in \WCal \times \Sigma$ be fixed, we have
\begin{equation*}
\frac{\left\langle\begin{pmatrix} \w-\w_0 \\ \sigma - \sigma_0 \end{pmatrix}, T\begin{pmatrix} \w_0 \\ \sigma_0 \end{pmatrix}\right\rangle_\HCal}{\sqrt{\|\w\|^2+\EE_q[\|\sigma(\cb_d)\|_F^2]}} \ \geq \ -C, \quad \textnormal{for some universal constant } C > 0, 
\end{equation*}
and 
\begin{equation*}
\frac{\|\w-\w_0\|^2 + \EE_q[\|\sigma(\cb_d)-\sigma_0(\cb_d)\|_F^2]}{\sqrt{\|\w\|^2+\EE_q[\|\sigma(\cb_d)\|_F^2]}} \ \longrightarrow \ +\infty. 
\end{equation*}
This implies that Eq.~\eqref{Condition:VI-existence} holds true, which completes the proof. 

\paragraph{Proof of Corollary~\ref{corollary:VI-existence-fourth}}
The proof is nearly the same as that of Corollary~\ref{corollary:VI-existence-third} and we need to verify the assumptions in Theorem~\ref{Theorem:existence-uniqueness}. Note that $T$ defined by Eq.~\eqref{Eq:mapping-infinite} is strictly monotone and $\WCal \times \Sigma$ is compact. Thus, by using Theorem~\ref{Theorem:existence-uniqueness}, we conclude the desired result.

\section{Postponed Proofs in Section~\ref{sec:alg_inertial}}
In this section, we provide the detailed proof for Theorem~\ref{Theorem:PRGIE}. We start by reviewing two preliminary results in the literature which are established as~\cite[Lemma~2.6]{Saejung-2012-Approximation} and the Minty's lemma in~\cite{Bauschke-2011-Convex} respectively. 
\begin{lemma}\label{lemma:sequence}
Let $\{s_t\}_{t \geq 0}$ be a sequence of nonnegative real numbers, $\{a_t\}_{t \geq 0}$ be a sequence in $(0, 1)$ such that $\sum_{t=0}^{+\infty} a_t = +\infty$ and $\{b_t\}_{t \geq 0}$ be a sequence of real numbers. Suppose that $s_{t+1} \leq (1-a_t)s_t + a_tb_t$ for all $t \geq 0$. If $\limsup_{j \rightarrow +\infty} b_{t_j} \leq 0$ for every subsequence $\{s_{t_j}\}_{j \geq 0}$ of $\{s_t\}_{t \geq 0}$ satisfying that $\liminf_{j \rightarrow +\infty} (s_{t_j+1} - s_{t_j}) \geq 0$, then $s_t \rightarrow 0$ as $t \rightarrow +\infty$. 
\end{lemma}
\begin{lemma}[Minty]\label{lemma:minty}
Let $T: \HCal \rightarrow \HCal$ be a continuous and monotone mapping on a closed and convex subset $\SCal$. Then $\widehat{\x}$ is a solution of the VI in Eq.~\eqref{prob:VI-general} if and only if $\langle \x - \widehat{\x}, T(\x)\rangle_\HCal \geq 0$ for all $\x \in \SCal$.  
\end{lemma}
First, we show that the sequences $\{(\widetilde{\w}_t, \widetilde{\sigma}_t)\}_{t \geq 1}$ and $\{(\w_t, \sigma_t)\}_{t \geq 1}$ generated by Algorithm~\ref{Algorithm:PRG-IE} are both bounded in the following lemma.  
\begin{lemma}\label{lemma:boundedness}
Under Assumption~\ref{Assumption:main} and~\ref{Assumption:PRG-IE}, there exists a constant $M>0$ such that the sequences $\{(\widetilde{\w}_t, \widetilde{\sigma}_t)\}_{t \geq 0}$ and $\{(\w_t, \sigma_t)\}_{t \geq 1}$ generated by Algorithm~\ref{Algorithm:PRG-IE} satisfies that
\begin{equation*}
\|\widetilde{\w}_t\|^2 + \EE_q[\|\widetilde{\sigma}_t(\cb_d)\|_F^2] \leq M \quad \textnormal{and} \quad \|\w_t\|^2 + \EE_q[\|\sigma_t(\cb_d)\|_F^2] \leq M \quad \textnormal{for all } t \geq 1. 
\end{equation*}
\end{lemma}
\begin{proof}
By the compactness of actions spaces $\WCal$ and $\XCal$, it suffices to show that $(\widetilde{\w}_t, \widetilde{\sigma}_t) \in \WCal \times \Sigma$ for all $t \geq 0$ and $(\w_t, \sigma_t) \in \WCal \times \Sigma$ for all $t \geq 1$. Indeed, the initialization step implies that $\widetilde{\w}_0, \widetilde{\w}_1, \w_1 \in \WCal$ and $\widetilde{\sigma}_0, \widetilde{\sigma}_1, \sigma_1 \in \Sigma$. Then, by the updating formula for $\{(\widetilde{\w}_t, \widetilde{\sigma}_t)\}_{t \geq 2}$, it is clear that for $(\widetilde{\w}_t, \widetilde{\sigma}_t) \in \WCal \times \Sigma$ for all $t \geq 0$. It remains to show that $(\w_t, \sigma_t) \in \WCal \times \Sigma$ for all $t \geq 2$. Since $\zero_{n \times m} \in \XCal$, we have $\zero_{n \times m}(\cdot) \in \Sigma$ where $\zero_{n \times m}(\cdot)$ is a constant function from $\br^n$ to $\br^{n \times m}$ with value $\zero_{n \times m}$. By the convexity of $\WCal$ and $\XCal$, we have the convexity of $\WCal \times \Sigma$. By the updating formula for $\{(\w_t, \sigma_t)\}_{t \geq 2}$, we find that $(\w_{t+1}, \sigma_{t+1})$ is a convex combination of $(\widetilde{\w}_{t+1}, \widetilde{\sigma}_{t+1})$, $(\w_t, \sigma_t)$ and $(\zero_m, \zero_{n \times m}(\cdot))$. This together with the convexity of $\WCal \times \Sigma$ implies the desired result. This completes the proof.
\end{proof}
We then present an important descent inequality for the sequences $\{(\widetilde{\w}_t, \widetilde{\sigma}_t)\}_{t \geq 0}$ and $\{(\w_t, \sigma_t)\}_{t \geq 1}$ generated by Algorithm~\ref{Algorithm:PRG-IE}. Note that our result can not be derived from~\cite[Lemma~3.1]{Malitsky-2015-Projected} due to the Halpern-type inertial extrapolation. 
\begin{lemma}\label{lemma:key-descent}
Under Assumption~\ref{Assumption:main} and~\ref{Assumption:PRG-IE}, the sequences $\{(\widetilde{\w}_t, \widetilde{\sigma}_t)\}_{t \geq 0}$ and $\{(\w_t, \sigma_t)\}_{t \geq 1}$ generated by Algorithm~\ref{Algorithm:PRG-IE} satisfies that
\begin{eqnarray*}
\|\widetilde{\w}_{t+1}-\w_\star\|^2+\EE_q[\|\widetilde{\sigma}_{t+1}(\cb_d)-\sigma_\star(\cb_d)\|_F^2] & \leq & \|\w_t-\w_\star\|^2 + \EE_q[\|\sigma_t(\cb_d)-\sigma_\star(\cb_d)\|_F^2] \\ 
& & \hspace{-20em} - \left(\frac{1}{2} - 2\gamma L\right)\left(\|\widetilde{\w}_{t+1}-\bar{\w}_t\|^2 + \EE_q[\|\widetilde{\sigma}_{t+1}(\cb_d)-\bar{\sigma}_t(\cb_d)\|_F^2]\right) + 6\gamma L\left(\|\widetilde{\w}_t-\bar{\w}_{t-1}\|^2 + \EE_q[\|\widetilde{\sigma}_t(\cb_d)-\bar{\sigma}_{t-1}(\cb_d)\|_F^2]\right) \\
& & \hspace{-20em} + \left(4 + 6\gamma L\right)\left(\|\w_t-\widetilde{\w}_t\|^2 + \EE_q[\|\sigma_t(\cb_d)-\widetilde{\sigma}_t(\cb_d)\|_F^2]\right) + 4\left(\|\w_{t-1}-\widetilde{\w}_{t-1}\|^2 + \EE_q[\|\sigma_{t-1}(\cb_d)-\widetilde{\sigma}_{t-1}(\cb_d)\|_F^2]\right) \\
& & \hspace{-20em} - 4\gamma \EE_q\left[\langle\nabla_\w\widehat{\theta}_l(\w_\star, \sigma_\star(\cb_d), \cb_l), \widetilde{\w}_t-\w_\star\rangle + \langle \nabla_{\bar{X}}\widehat{\theta}_d(\w_\star, \sigma_\star(\cb_d), \cb_d), \widetilde{\sigma}_t(\cb_d)-\sigma_\star(\cb_d)\rangle\right] \\
& & \hspace{-20em} + 2\gamma\EE_q\left[\langle\nabla_\w\widehat{\theta}_l(\w_\star, \sigma_\star(\cb_d), \cb_l), \widetilde{\w}_{t-1}-\w_\star\rangle + \langle \nabla_{\bar{X}}\widehat{\theta}_d(\w_\star, \sigma_\star(\cb_d), \cb_d), \widetilde{\sigma}_{t-1}(\cb_d)-\sigma_\star(\cb_d)\rangle\right] \\
& & \hspace{-20em} - \left(1 - 6\gamma L\right)\left(\|\bar{\w}_t-\w_t\|^2 + \EE_q[\|\bar{\sigma}_t(\cb_d)-\sigma_t(\cb_d)\|_F^2]\right), 
\end{eqnarray*}
where the point $(\w_\star, \sigma_\star) \in \WCal \times \Sigma$ is a unique Bayesian equilibrium, and the auxiliary sequence $(\bar{\w}_t, \bar{\sigma}_t)$ is defined by $(\bar{\w}_t, \bar{\sigma}_t) = (2\widetilde{\w}_t-\widetilde{\w}_{t-1}, 2\tilde{\sigma}_t-\widetilde{\sigma}_{t-1})$.
\end{lemma}
\begin{proof}
Under Assumption~\ref{Assumption:main} and~\ref{Assumption:PRG-IE}, a unique Bayesian equilibrium $(\w_\star, \sigma_\star) \in \WCal \times \Sigma$ exists. Given that the sequence $(\bar{\w}_t, \bar{\sigma}_t) = (2\widetilde{\w}_t-\widetilde{\w}_{t-1}, 2\tilde{\sigma}_t-\widetilde{\sigma}_{t-1})$ is defined, we derive from the optimality condition of updating $(\widetilde{\w}_{t+1}, \widetilde{\sigma}_{t+1})$ and the definition of $P_\WCal$ and $P_\Sigma$ that,
\begin{eqnarray}\label{opt:PRGIE}
0 & \leq & \EE_q\left[\langle \w - \widetilde{\w}_{t+1}, \widetilde{\w}_{t+1} - \w_t + \gamma\nabla_\w\widehat{\theta}_l(\bar{\w}_t, \bar{\sigma}_t(\cb_d), \cb_l)\rangle\right. \quad \textnormal{for each } (\w, \sigma) \in \WCal \times \Sigma, \\
& & \left. + \langle \sigma(\cb_d) - \widetilde{\sigma}_{t+1}(\cb_d), \widetilde{\sigma}_{t+1}(\cb_d) - \sigma_t(\cb_d) + \gamma\nabla_{\bar{X}}\widehat{\theta}_d(\bar{\w}_t, \bar{\sigma}_t(\cb_d), \cb_d)\rangle\right],  \nonumber
\end{eqnarray}
By rearranging the terms in the above inequality with $(\w, \sigma) = (\w_\star, \sigma_\star) \in \WCal \times \Sigma$ and the equality $\langle a,b\rangle = (1/2)(\|a+b\|^2-\|a\|^2-\|b\|^2)$, we have
\begin{eqnarray}\label{inequality:PRGIE-key-first}
\|\widetilde{\w}_{t+1}-\w_\star\|^2+\EE_q[\|\widetilde{\sigma}_{t+1}(\cb_d)-\sigma_\star(\cb_d)\|_F^2] & \leq & \|\w_t-\w_\star\|^2 + \EE_q[\|\sigma_t(\cb_d)-\sigma_\star(\cb_d)\|_F^2] \\ 
& & \hspace{-18em} - \|\widetilde{\w}_{t+1}-\w_t\|^2 - \EE_q[\|\widetilde{\sigma}_{t+1}(\cb_d)-\sigma_t(\cb_d)\|_F^2] - 2\gamma\EE_q\left[\langle\nabla_\w\widehat{\theta}_l(\bar{\w}_t, \bar{\sigma}_t(\cb_d), \cb_l), \widetilde{\w}_{t+1} - \w_\star\rangle \right. \nonumber \\ 
& & \hspace{-18em} \left. + \langle \nabla_{\bar{X}}\widehat{\theta}_d(\bar{\w}_t, \bar{\sigma}_t(\cb_d), \cb_d), \widetilde{\sigma}_{t+1}(\cb_d) - \sigma_\star(\cb_d)\rangle\right]. \nonumber
\end{eqnarray}
Moreover, we have
\begin{eqnarray*}
\lefteqn{\EE_q\left[\langle\nabla_\w\widehat{\theta}_l(\bar{\w}_t, \bar{\sigma}_t(\cb_d), \cb_l), \widetilde{\w}_{t+1} - \w_\star\rangle + \langle \nabla_{\bar{X}}\widehat{\theta}_d(\bar{\w}_t, \bar{\sigma}_t(\cb_d), \widetilde{\sigma}_{t+1}(\cb_d) - \sigma_\star(\cb_d)\rangle\right]} \\ 
& =  \EE_q\left[\langle\nabla_\w\widehat{\theta}_l(\bar{\w}_t, \bar{\sigma}_t(\cb_d), \cb_l), \widetilde{\w}_{t+1} - \bar{\w}_t\rangle + \langle \nabla_{\bar{X}}\widehat{\theta}_d(\bar{\w}_t, \bar{\sigma}_t(\cb_d), \cb_d), \widetilde{\sigma}_{t+1}(\cb_d) - \bar{\sigma}_t(\cb_d)\rangle\right] \\ 
& + \EE_q\left[\langle\nabla_\w\widehat{\theta}_l(\bar{\w}_t, \bar{\sigma}_t(\cb_d), \cb_l), \bar{\w}_t - \w_\star\rangle + \langle \nabla_{\bar{X}}\widehat{\theta}_d(\bar{\w}_t, \bar{\sigma}_t(\cb_d), \cb_d), \bar{\sigma}_t(\cb_d) - \sigma_\star(\cb_d)\rangle\right]. 
\end{eqnarray*}
Using the first condition in Assumption~\ref{Assumption:PRG-IE} with $(\w, \sigma) = (\bar{\w}_t, \bar{\sigma}_t)$ and $(\w', \sigma')=(\w_\star, \sigma_\star)$, we have
\begin{eqnarray*}
\lefteqn{\EE_q\left[\langle\nabla_\w\widehat{\theta}_l(\bar{\w}_t, \bar{\sigma}_t(\cb_d), \cb_l), \bar{\w}_t-\w_\star\rangle + \langle \nabla_{\bar{X}}\widehat{\theta}_d(\bar{\w}_t, \bar{\sigma}_t(\cb_d), \cb_d), \bar{\sigma}_t(\cb_d)-\sigma_\star(\cb_d)\rangle\right]} \\
& \geq & \EE_q\left[\langle\nabla_\w\widehat{\theta}_l(\w_\star, \sigma_\star(\cb_d), \cb_l), \bar{\w}_t-\w_\star\rangle + \langle\nabla_{\bar{X}}\widehat{\theta}_d(\w_\star, \sigma_\star(\cb_d), \cb_d), \bar{\sigma}_t(\cb_d) - \sigma_\star(\cb_d)\rangle\right] \\ 
& = & 2\EE_q\left[\langle\nabla_\w\widehat{\theta}_l(\w_\star, \sigma_\star(\cb_d), \cb_l), \widetilde{\w}_t-\w_\star\rangle + \langle \nabla_{\bar{X}}\widehat{\theta}_d(\w_\star, \sigma_\star(\cb_d), \cb_d), \widetilde{\sigma}_t(\cb_d)-\sigma_\star(\cb_d)\rangle\right] \\
& & - \EE_q\left[\langle\nabla_\w\widehat{\theta}_l(\w_\star, \sigma_\star(\cb_d), \cb_l), \widetilde{\w}_{t-1}-\w_\star\rangle + \langle \nabla_{\bar{X}}\widehat{\theta}_d(\w_\star, \sigma_\star(\cb_d), \cb_d), \widetilde{\sigma}_{t-1}(\cb_d)-\sigma_\star(\cb_d)\rangle\right].  
\end{eqnarray*}
Putting these two inequalities together with Eq.~\eqref{inequality:PRGIE-key-first} yields that 
\begin{eqnarray}\label{inequality:PRGIE-key-second}
\|\widetilde{\w}_{t+1}-\w_\star\|^2+\EE_q[\|\widetilde{\sigma}_{t+1}(\cb_d)-\sigma_\star(\cb_d)\|_F^2] & \leq & \|\w_t-\w_\star\|^2 + \EE_q[\|\sigma_t(\cb_d)-\sigma_\star(\cb_d)\|_F^2] \\ 
& & \hspace{-18em} - \|\widetilde{\w}_{t+1}-\w_t\|^2 - \EE_q[\|\widetilde{\sigma}_{t+1}(\cb_d)-\sigma_t(\cb_d)\|_F^2] \nonumber \\ 
& & \hspace{-18em} - 2\gamma\EE_q\left[\langle\nabla_\w\widehat{\theta}_l(\bar{\w}_t, \bar{\sigma}_t(\cb_d), \cb_l)-\nabla_\w\widehat{\theta}_l(\bar{\w}_{t-1}, \bar{\sigma}_{t-1}(\cb_d), \cb_l), \widetilde{\w}_{t+1}-\bar{\w}_t\rangle\right] \nonumber \\
& & \hspace{-18em} - 2\gamma\EE_q\left[\langle \nabla_{\bar{X}}\widehat{\theta}_d(\bar{\w}_t, \bar{\sigma}_t(\cb_d), \cb_d)-\nabla_{\bar{X}}\widehat{\theta}_d(\bar{\w}_{t-1}, \bar{\sigma}_{t-1}(\cb_d), \cb_d), \widetilde{\sigma}_{t+1}(\cb_d)-\bar{\sigma}_t(\cb_d)\rangle\right] \nonumber \\
& & \hspace{-18em} - 2\gamma\EE_q\left[\langle\nabla_\w\widehat{\theta}_l(\bar{\w}_{t-1}, \bar{\sigma}_{t-1}(\cb_d), \cb_l), \widetilde{\w}_{t+1}-\bar{\w}_t\rangle + \langle \nabla_{\bar{X}}\widehat{\theta}_d(\bar{\w}_{t-1}, \bar{\sigma}_{t-1}(\cb_d), \cb_d), \widetilde{\sigma}_{t+1}(\cb_d)-\bar{\sigma}_t(\cb_d)\rangle\right] \nonumber \\
& & \hspace{-18em} - 4\gamma \EE_q\left[\langle\nabla_\w\widehat{\theta}_l(\w_\star, \sigma_\star(\cb_d), \cb_l), \widetilde{\w}_t-\w_\star\rangle + \langle \nabla_{\bar{X}}\widehat{\theta}_d(\w_\star, \sigma_\star(\cb_d), \cb_d), \widetilde{\sigma}_t(\cb_d)-\sigma_\star(\cb_d)\rangle\right] \nonumber \\
& & \hspace{-18em} + 2\gamma\EE_q\left[\langle\nabla_\w\widehat{\theta}_l(\w_\star, \sigma_\star(\cb_d), \cb_l), \w_{t-1}-\w_\star\rangle + \langle \nabla_{\bar{X}}\widehat{\theta}_d(\w_\star, \sigma_\star(\cb_d), \cb_d), \sigma_{t-1}(\cb_d)-\sigma_\star(\cb_d)\rangle\right]. \nonumber 
\end{eqnarray}
Then, we obtain two inequalities by changing the index $t$ in Eq.~\eqref{opt:PRGIE} to $t-1$ and letting $(\w, \sigma) = (\widetilde{\w}_{t-1}, \widetilde{\sigma}_{t-1})$ and $(\w, \sigma) = (\widetilde{\w}_{t+1}, \widetilde{\sigma}_{t+1})$ and add them to obtain that 
\begin{eqnarray*}
0 & \leq & \EE_q\left[\langle \widetilde{\w}_{t+1}-\bar{\w}_t, \widetilde{\w}_t - \w_{t-1} + \gamma\nabla_\w\widehat{\theta}_l(\bar{\w}_{t-1}, \bar{\sigma}_{t-1}(\cb_d), \cb_l)\rangle\right.  \\
& & \left. + \langle \widetilde{\sigma}_{t+1}(\cb_d)-\bar{\sigma}_t(\cb_d), \widetilde{\sigma}_t(\cb_d) - \sigma_{t-1}(\cb_d) + \gamma\nabla_{\bar{X}}\widehat{\theta}_d(\bar{\w}_{t-1}, \bar{\sigma}_{t-1}(\cb_d), \cb_d)\rangle \right].    
\end{eqnarray*}
By the definition of the sequence $(\bar{\w}_t, \bar{\sigma}_t)$, we further have
\begin{eqnarray*}
\lefteqn{- 2\gamma\EE_q\left[\langle\nabla_\w\widehat{\theta}_l(\bar{\w}_{t-1}, \bar{\sigma}_{t-1}(\cb_d), \cb_l), \widetilde{\w}_{t+1}-\bar{\w}_t\rangle + \langle \nabla_{\bar{X}}\widehat{\theta}_d(\bar{\w}_{t-1}, \bar{\sigma}_{t-1}(\cb_d), \cb_d), \widetilde{\sigma}_{t+1}(\cb_d)-\bar{\sigma}_t(\cb_d)\rangle\right]} \\ 
& \leq & 2\langle\widetilde{\w}_{t+1}-\bar{\w}_t, \widetilde{\w}_t - \w_{t-1}\rangle + 2\EE_q[\langle \widetilde{\sigma}_{t+1}(\cb_d)-\bar{\sigma}_t(\cb_d), \widetilde{\sigma}_t(\cb_d) - \sigma_{t-1}(\cb_d)\rangle] \\ 
& = & 2\langle \widetilde{\w}_{t+1}-\bar{\w}_t, \bar{\w}_t-\widetilde{\w}_t\rangle + 2\langle \widetilde{\w}_{t+1}-\bar{\w}_t, \widetilde{\w}_{t-1} - \w_{t-1}\rangle + 2\EE_q[\langle \widetilde{\sigma}_{t+1}(\cb_d)-\bar{\sigma}_t(\cb_d), \bar{\sigma}_t(\cb_d)-\widetilde{\sigma}_t(\cb_d)\rangle] \\
& & + 2\EE_q[\langle \widetilde{\sigma}_{t+1}(\cb_d)-\bar{\sigma}_t(\cb_d), \widetilde{\sigma}_{t-1}(\cb_d)-\sigma_{t-1}(\cb_d)\rangle] \\ 
& = & 2\langle\widetilde{\w}_{t+1}-\bar{\w}_t, \bar{\w}_t-\w_t\rangle + 2\EE_q[\langle \widetilde{\sigma}_{t+1}(\cb_d)-\bar{\sigma}_t(\cb_d), \bar{\sigma}_t(\cb_d)-\sigma_t(\cb_d)\rangle] + 2\langle \widetilde{\w}_{t+1}-\bar{\w}_t, \w_t-\widetilde{\w}_t\rangle \\ 
& & + 2\EE_q[\langle \widetilde{\sigma}_{t+1}(\cb_d)-\bar{\sigma}_t(\cb_d), \sigma_t(\cb_d)-\widetilde{\sigma}_t(\cb_d)\rangle] + 2\langle \widetilde{\w}_{t+1}-\bar{\w}_t, \widetilde{\w}_{t-1} - \w_{t-1}\rangle \\ 
& & + 2\EE_q\left[\langle \widetilde{\sigma}_{t+1}(\cb_d)-\bar{\sigma}_t(\cb_d), \widetilde{\sigma}_{t-1}(\cb_d)-\sigma_{t-1}(\cb_d)\rangle \right]. 
\end{eqnarray*}
Using the equality $\langle a,b\rangle = (1/2)(\|a+b\|^2-\|a\|^2-\|b\|^2)$ for the first two terms, and the Young's inequality $\langle a,b\rangle \leq (1/8)\|a\|^2 + 2\|b\|^2$ for the last four terms, we have
\begin{eqnarray}\label{inequality:PRGIE-key-third}
\lefteqn{- 2\gamma\EE_q\left[\langle\nabla_\w\widehat{\theta}_l(\bar{\w}_{t-1}, \bar{\sigma}_{t-1}(\cb_d), \cb_l), \widetilde{\w}_{t+1}-\bar{\w}_t\rangle + \langle \nabla_{\bar{X}}\widehat{\theta}_d(\bar{\w}_{t-1}, \bar{\sigma}_{t-1}(\cb_d), \cb_d), \widetilde{\sigma}_{t+1}(\cb_d)-\bar{\sigma}_t(\cb_d)\rangle\right]} \nonumber \\ 
& \leq & \|\widetilde{\w}_{t+1}-\w_t\|^2 + \EE_q[\|\widetilde{\sigma}_{t+1}(\cb_d)-\sigma_t(\cb_d)\|_F^2] - \frac{1}{2}\left(\|\widetilde{\w}_{t+1}-\bar{\w}_t\|^2 + \EE_q[\|\widetilde{\sigma}_{t+1}(\cb_d)-\bar{\sigma}_t(\cb_d)\|_F^2]\right) \nonumber \\ 
& & - \left(\|\bar{\w}_t-\w_t\|^2 + \EE_q[\|\bar{\sigma}_t(\cb_d)-\sigma_t(\cb_d)\|_F^2]\right) + 4\left(\|\w_t-\widetilde{\w}_t\|^2 + \EE_q[\|\sigma_t(\cb_d)-\widetilde{\sigma}_t(\cb_d)\|_F^2]\right) \nonumber \\
& & + 4\left(\|\w_{t-1}-\widetilde{\w}_{t-1}\|^2 + \EE_q[\|\sigma_{t-1}(\cb_d)-\widetilde{\sigma}_{t-1}(\cb_d)\|_F^2]\right).  
\end{eqnarray} 
By using the second condition in Assumption~\ref{Assumption:PRG-IE}, we have
\begin{eqnarray}\label{inequality:PRGIE-key-fourth}
\lefteqn{- 2\gamma\EE_q\left[\langle\nabla_\w\widehat{\theta}_l(\bar{\w}_t, \bar{\sigma}_t(\cb_d), \cb_l)-\nabla_\w\widehat{\theta}_l(\bar{\w}_{t-1}, \bar{\sigma}_{t-1}(\cb_d), \cb_l), \widetilde{\w}_{t+1}-\bar{\w}_t\rangle \right.} \nonumber \\
& & \left. + \langle \nabla_{\bar{X}}\widehat{\theta}_d(\bar{\w}_t, \bar{\sigma}_t(\cb_d), \cb_d)-\nabla_{\bar{X}}\widehat{\theta}_d(\bar{\w}_{t-1}, \bar{\sigma}_{t-1}(\cb_d), \cb_d), \widetilde{\sigma}_{t+1}(\cb_d)-\bar{\sigma}_t(\cb_d)\rangle\right] \\
& \leq & 2\gamma L\left(\|\bar{\w}_t-\bar{\w}_{t-1}\| + \EE_q[\|\bar{\sigma}_t(\cb_d)-\bar{\sigma}_{t-1}(\cb_d)\|_F]\right)\left(\|\widetilde{\w}_{t+1}-\bar{\w}_t\| + \EE_q[\|\widetilde{\sigma}_{t+1}(\cb_d)-\bar{\sigma}_t(\cb_d)\|_F]\right) \nonumber \\
& \leq & 2\gamma L\left(\|\bar{\w}_t-\bar{\w}_{t-1}\|^2 + \EE_q[\|\bar{\sigma}_t(\cb_d)-\bar{\sigma}_{t-1}(\cb_d)\|_F^2] + \|\widetilde{\w}_{t+1}-\bar{\w}_t\|^2 + \EE_q[\|\widetilde{\sigma}_{t+1}(\cb_d)-\bar{\sigma}_t(\cb_d)\|_F^2]\right) \nonumber \\
& \leq & 2\gamma L\left(\|\widetilde{\w}_{t+1}-\bar{\w}_t\|^2 + \EE_q[\|\widetilde{\sigma}_{t+1}(\cb_d)-\bar{\sigma}_t(\cb_d)\|_F^2]\right) + 6\gamma L\left(\|\bar{\w}_t-\w_t\|^2 + \EE_q[\|\bar{\sigma}_t(\cb_d)-\sigma_t(\cb_d)\|_F^2]\right) \nonumber \\
& & + 6\gamma L\left(\|\w_t-\widetilde{\w}_t\|^2 + \EE_q[\|\sigma_t(\cb_d)-\widetilde{\sigma}_t(\cb_d)\|_F^2] + \|\widetilde{\w}_t-\bar{\w}_{t-1}\|^2 + \EE_q[\|\widetilde{\sigma}_t(\cb_d)-\bar{\sigma}_{t-1}(\cb_d)\|_F^2]\right). \nonumber
\end{eqnarray}
Combining Eq.~\eqref{inequality:PRGIE-key-second}, Eq.~\eqref{inequality:PRGIE-key-third} and Eq.~\eqref{inequality:PRGIE-key-fourth} yields the desired inequality. 
\end{proof}

Based on the boundedness results in Lemma~\ref{lemma:boundedness} and the descent inequality in Lemma~\ref{lemma:key-descent}, we provide two additional inequalities for the sequences $\{(\widetilde{\w}_t, \widetilde{\sigma}_t)\}_{t \geq 0}$ and $\{(\w_t, \sigma_t)\}_{t \geq 1}$ generated by Algorithm~\ref{Algorithm:PRG-IE} in the following lemma.  

\begin{lemma}\label{lemma:PRGIE}
Under Assumption~\ref{Assumption:main} and~\ref{Assumption:PRG-IE}, the sequences $\{(\widetilde{\w}_t, \widetilde{\sigma}_t)\}_{t \geq 0}$ and $\{(\w_t, \sigma_t)\}_{t \geq 1}$ generated by Algorithm~\ref{Algorithm:PRG-IE} satisfies that
\begin{align}
r_{t+1} & \leq \ r_t + \delta_{t-1}M_1 - \left(\frac{1}{2} - 8\gamma L\right)\left(\|\widetilde{\w}_{t+1}-\bar{\w}_t\|^2 + \EE_q[\|\widetilde{\sigma}_{t+1}(\cb_d)-\bar{\sigma}_t(\cb_d)\|_F^2]\right) \label{inequality:PRGIE-first} \\
& \hspace{1em} - \left(1 - 6\gamma L\right)\left(\|\bar{\w}_t-\w_t\|^2 + \EE_q[\|\bar{\sigma}_t(\cb_d)-\sigma_t(\cb_d)\|_F^2]\right), \nonumber \\
r_{t+1} & \leq \ \left(1-\frac{3\delta_t}{4}\right)r_t + \frac{3\delta_t}{4}\left(4\delta_t M_1 - 2\langle \w_\star, \w_{t+1} - \w_\star\rangle - 2\EE_q[\langle \sigma_\star(\cb_d), \sigma_{t+1}(\cb_d)-\sigma_\star(\cb_d)\rangle]\right), \label{inequality:PRGIE-second} \\
& \hspace{1em} \textnormal{for all } t \geq 2 \textnormal{ and some constant } M_1 > 0. \nonumber
\end{align}
The residue sequence $\{r_t\}_{t \geq 0}$ is defined as:
\begin{eqnarray*}
r_t & = & \|\w_t-\w_\star\|^2 + \EE_q[\|\sigma_t(\cb_d)-\sigma_\star(\cb_d)\|_F^2] + 6\gamma L\left(\|\widetilde{\w}_t-\bar{\w}_{t-1}\|^2 + \EE_q[\|\widetilde{\sigma}_t(\cb_d)-\bar{\sigma}_{t-1}(\cb_d)\|_F^2]\right) \\
& & + 2\gamma\EE_q\left[\langle\nabla_\w\widehat{\theta}_l(\w_\star, \sigma_\star(\cb_d), \cb_l), \widetilde{\w}_{t-1}-\w_\star\rangle + \langle \nabla_{\bar{X}}\widehat{\theta}_d(\w_\star, \sigma_\star(\cb_d), \cb_d), \widetilde{\sigma}_{t-1}(\cb_d)-\sigma_\star(\cb_d)\rangle\right], 
\end{eqnarray*}
where $(\w_\star, \sigma_\star) \in \WCal \times \Sigma$ is a unique Bayesian equilibrium, and the auxiliary sequence $(\bar{\w}_t, \bar{\sigma}_t)$ is defined in Lemma~\ref{lemma:key-descent}. 
\end{lemma}
\begin{proof}
By the updating formula for the sequence $\{(\w_t, \sigma_t)\}_{t \geq 1}$ and the Jensen's inequality, we have
\begin{eqnarray}\label{inequality:PRGIE-main-first}
& & \|\w_{t+1}-\w_\star\|^2 + \EE_q[\|\sigma_{t+1}(\cb_d)-\sigma_\star(\cb_d)\|_F^2] \\
& \leq & \frac{\delta_t}{2}\left(\|\w_t-\w_\star\|^2 + \EE_q[\|\sigma_t(\cb_d)-\sigma_\star(\cb_d)\|_F^2]\right) + \frac{\delta_t}{2}\left(\|\w_\star\|^2 + \EE_q[\|\sigma_\star(\cb_d)\|_F^2]\right) \nonumber \\
& & + \left(1-\delta_t\right)\left(\|\widetilde{\w}_{t+1}-\w_\star\|^2+\EE_q[\|\widetilde{\sigma}_{t+1}(\cb_d)-\sigma_\star(\cb_d)\|_F^2]\right). \nonumber 
\end{eqnarray}
Recall that the residue sequence $r_t \geq 0$ is defined by
\begin{eqnarray*}
r_t & = & \|\w_t-\w_\star\|^2 + \EE_q[\|\sigma_t(\cb_d)-\sigma_\star(\cb_d)\|_F^2] + 6\gamma L\left(\|\widetilde{\w}_t-\bar{\w}_{t-1}\|^2 + \EE_q[\|\widetilde{\sigma}_t(\cb_d)-\bar{\sigma}_{t-1}(\cb_d)\|_F^2]\right) \\
& & + 2\gamma\EE_q\left[\langle\nabla_\w\widehat{\theta}_l(\w_\star, \sigma_\star(\cb_d), \cb_l), \widetilde{\w}_{t-1}-\w_\star\rangle + \langle \nabla_{\bar{X}}\widehat{\theta}_d(\w_\star, \sigma_\star(\cb_d), \cb_d), \widetilde{\sigma}_{t-1}(\cb_d)-\sigma_\star(\cb_d)\rangle\right]. 
\end{eqnarray*}
Combining this with Eq.~\eqref{inequality:PRGIE-main-first} and Lemma~\ref{lemma:key-descent}, we have
\begin{eqnarray}\label{inequality:PRGIE-main-second}
r_{t+1} & \leq & r_t - \left(\frac{1}{2} - 8\gamma L\right)\left(\|\widetilde{\w}_{t+1}-\bar{\w}_t\|^2 + \EE_q[\|\widetilde{\sigma}_{t+1}(\cb_d)-\bar{\sigma}_t(\cb_d)\|_F^2]\right) + \frac{\delta_t}{2}\left(\|\w_\star\|^2 + \EE_q[\|\sigma_\star(\cb_d)\|_F^2]\right) \nonumber \\
& & \hspace{-2em} + \left(4 + 6\gamma L\right)\left(\|\w_t-\widetilde{\w}_t\|^2 + \EE_q[\|\sigma_t(\cb_d)-\widetilde{\sigma}_t(\cb_d)\|_F^2]\right) + 4\left(\|\w_{t-1}-\widetilde{\w}_{t-1}\|^2 + \EE_q[\|\sigma_{t-1}(\cb_d)-\widetilde{\sigma}_{t-1}(\cb_d)\|_F^2]\right) \nonumber \\ 
& & \hspace{-2em} - 2\gamma \EE_q\left[\langle\nabla_\w\widehat{\theta}_l(\w_\star, \sigma_\star(\cb_d), \cb_l), \widetilde{\w}_t-\w_\star\rangle + \langle \nabla_{\bar{X}}\widehat{\theta}_d(\w_\star, \sigma_\star(\cb_d), \cb_d), \widetilde{\sigma}_t(\cb_d)-\sigma_\star(\cb_d)\rangle\right] \nonumber \\
& & \hspace{-2em} - \left(1 - 6\gamma L\right)\left(\|\bar{\w}_t-\w_t\|^2 + \EE_q[\|\bar{\sigma}_t(\cb_d)-\sigma_t(\cb_d)\|_F^2]\right).    
\end{eqnarray}
Since $(\w_\star, \sigma_\star) \in \WCal \times \Sigma$ is a unique solution of the VI in Eq.~\eqref{prob:VI-infinite} under Assumption~\ref{Assumption:main} and~\ref{Assumption:PRG-IE}, we have
\begin{equation*}
\EE_q\left[\langle\nabla_\w\widehat{\theta}_l(\w_\star, \sigma_\star(\cb_d), \cb_l), \widetilde{\w}_t-\w_\star\rangle + \langle \nabla_{\bar{X}}\widehat{\theta}_d(\w_\star, \sigma_\star(\cb_d), \cb_d), \widetilde{\sigma}_t(\cb_d)-\sigma_\star(\cb_d)\rangle\right] \geq 0. 
\end{equation*}
Using Lemma~\ref{lemma:boundedness}, the updating formula for the sequence $\{(\w_t, \sigma_t)\}_{t \geq 1}$ and the uniqueness of $(\w_\star, \sigma_\star) \in \WCal \times \Sigma$, there exists a constant $M_1 > 0$ such that 
\begin{eqnarray*}
\|\w_\star\|^2 + \EE_q[\|\sigma_\star(\cb_d)\|_F^2] & \leq & M_1, \\
\|\w_t-\widetilde{\w}_t\|^2 + \EE_q[\|\sigma_t(\cb_d)-\widetilde{\sigma}_t(\cb_d)\|_F^2] & \leq & \frac{\delta_{t-1}^2M_1}{18} \ \leq \ \frac{\delta_{t-1}M_1}{18}.  
\end{eqnarray*}
Putting these pieces together with Eq.~\eqref{inequality:PRGIE-main-second} and the facts that $0 < \gamma < \min\{1, \frac{1}{100L}\}$ and the sequence $\{\delta_t\}_{t \geq 1}$ is non-increasing yields Eq.~\eqref{inequality:PRGIE-first}. Then, we proceed to prove Eq.~\eqref{inequality:PRGIE-second}. Using the inequality $\|a+b\|^2 \leq \|a\|^2+2\langle a+b, b\rangle$, together with the Jensen's inequality and the updating formula for the sequence $\{(\w_t, \sigma_t)\}_{t \geq 1}$ yields that 
\begin{eqnarray*}
& & \|\w_{t+1}-\w_\star\|^2 + \EE_q[\|\sigma_{t+1}(\cb_d)-\sigma_\star(\cb_d)\|_F^2] \\
& = & \|\delta_t(\w_t/2-\w_\star/2) + (1-\delta_t)(\widetilde{\w}_{t+1}-\w_\star) - \delta_t(\w_\star/2)\|^2 \\
& & + \EE_q[\|\delta_t(\sigma_t(\cb_d)/2-\sigma_\star(\cb_d)/2) + (1-\delta_t)(\widetilde{\sigma}_{t+1}(\cb_d)-\sigma_\star(\cb_d)) - \delta_t(\sigma_\star(\cb_d)/2)\|_F^2] \\
& \leq & \frac{\delta_t}{4}\left(\|\w_t-\w_\star\|^2 + \EE_q[\|\sigma_t(\cb_d)-\sigma_\star(\cb_d)\|_F^2]\right) + \left(1-\delta_t\right)\left(\|\widetilde{\w}_{t+1}-\w_\star\|^2+\EE_q[\|\widetilde{\sigma}_{t+1}(\cb_d)-\sigma_\star(\cb_d)\|_F^2]\right) \nonumber \\
& & - \delta_t\left(\langle \w_\star, \w_{t+1} - \w_\star\rangle + \EE_q[\langle \sigma_\star(\cb_d), \sigma_{t+1}(\cb_d)-\sigma_\star(\cb_d)\rangle]\right). \nonumber 
\end{eqnarray*}
Note that $1 - \delta_t \geq 1/2$ for all $t \geq 2$ and $0 < \gamma < \min\{1, \frac{1}{100L}\}$. By combining the above inequality and Lemma~\ref{lemma:key-descent}, we have
\begin{eqnarray}\label{inequality:PRGIE-main-third}
r_{t+1} & \leq & \left(1-\frac{3}{4}\delta_t\right)r_t - \delta_t\left(\langle \w_\star, \w_{t+1} - \w_\star\rangle + \EE_q[\langle \sigma_\star(\cb_d), \sigma_{t+1}(\cb_d)-\sigma_\star(\cb_d)\rangle]\right) \\ 
& & - \left(\frac{1}{4} - 7\gamma L\right)\left(\|\widetilde{\w}_{t+1}-\bar{\w}_t\|^2 + \EE_q[\|\widetilde{\sigma}_{t+1}(\cb_d)-\bar{\sigma}_t(\cb_d)\|_F^2]\right) \nonumber \\ 
& & + \left(4 + 6\gamma L\right)\left(\|\w_t-\widetilde{\w}_t\|^2 + \EE_q[\|\sigma_t(\cb_d)-\widetilde{\sigma}_t(\cb_d)\|_F^2]\right) \nonumber \\
& & + 4\left(\|\w_{t-1}-\widetilde{\w}_{t-1}\|^2 + \EE_q[\|\sigma_{t-1}(\cb_d)-\widetilde{\sigma}_{t-1}(\cb_d)\|_F^2]\right) \nonumber \\ 
& & - \left(\frac{1}{2} - 3\gamma L\right)\left(\|\bar{\w}_t-\w_t\|^2 + \EE_q[\|\bar{\sigma}_t(\cb_d)-\sigma_t(\cb_d)\|_F^2]\right), \quad \textnormal{for all } t \geq 2. \nonumber    
\end{eqnarray}
Note that $\delta_t \leq \delta_{t-1} \leq 2\delta_t$ for all $t \geq 2$ and there exists a constant $M_1 > 0$ such that 
\begin{equation*}
\|\w_t-\widetilde{\w}_t\|^2 + \EE_q[\|\sigma_t(\cb_d)-\widetilde{\sigma}_t(\cb_d)\|_F^2] \leq \frac{\delta_{t-1}^2M_1}{18} \leq \frac{2\delta_t^2M_1}{9}. 
\end{equation*}
Putting the above inequality and the fact that $0 < \gamma < \min\{1, \frac{1}{100L}\}$ together with Eq.~\eqref{inequality:PRGIE-main-third} yields that 
\begin{eqnarray*}
r_{t+1} & \leq & \left(1-\frac{3}{4}\delta_t\right)r_t + 2\delta_t^2 M_1 - \delta_t\left(\langle \w_\star, \w_{t+1} - \w_\star\rangle + \EE_q[\langle \sigma_\star(\cb_d), \sigma_{t+1}(\cb_d)-\sigma_\star(\cb_d)\rangle]\right) \\ 
& = & \left(1-\frac{3\delta_t}{4}\right)r_t + \frac{3\delta_t}{4}\left( 4\delta_t M_1 - 2\langle \w_\star, \w_{t+1} - \w_\star\rangle - 2\EE_q[\langle \sigma_\star(\cb_d), \sigma_{t+1}(\cb_d)-\sigma_\star(\cb_d)\rangle]\right).  
\end{eqnarray*}
This completes the proof.
\end{proof}
\textbf{Proof of Theorem~\ref{Theorem:PRGIE}.} It suffices to prove that $r_t \rightarrow 0$ as $t \rightarrow +\infty$. Suppose that $\{r_{t_j}\}_{j \geq 0}$ is any of the subsequences of the whole sequence $\{r_t\}_{t \geq 0}$ and satisfies that $\liminf_{j \rightarrow +\infty} \left(r_{t_j+1} - r_{t_j}\right) \geq 0$. From Eq.~\eqref{inequality:PRGIE-first} in Lemma~\ref{lemma:PRGIE}, we have
\begin{eqnarray*}
& & \limsup_{j \rightarrow +\infty} \left[\left(\frac{1}{2} - 8\gamma L\right) \left(\|\widetilde{\w}_{t_j+1}-\bar{\w}_{t_j}\|^2 + \EE_q[\|\widetilde{\sigma}_{t_j+1}(\cb_d)-\bar{\sigma}_{t_j}(\cb_d)\|_F^2]\right) \right. \\ 
& & \left. + \left(1 - 6\gamma L\right)\left(\|\bar{\w}_{t_j}-\w_{t_j}\|^2 + \EE_q[\|\bar{\sigma}_{t_j}(\cb_d)-\sigma_{t_j}(\cb_d)\|_F^2]\right)\right]  \\ 
& &\leq  \limsup_{j \rightarrow +\infty} \left(r_{t_j} - r_{t_j+1} + \delta_{t_j}M_1\right) \ \leq \ \limsup_{j \rightarrow +\infty} \left(r_{t_j} - r_{t_j+1}\right) + \limsup_{j \rightarrow +\infty} \delta_{t_j}M_1 \\ 
& &\leq  -\liminf_{j \rightarrow +\infty} \left(r_{t_j+1} - r_{t_j}\right) + \limsup_{j \rightarrow +\infty} \delta_{t_j}M_1 \ \leq \ \limsup_{j \rightarrow +\infty} \delta_{t_j}M_1.  
\end{eqnarray*}
Since $\delta_t \rightarrow 0$ and $0 < \gamma < \min\{1, \frac{1}{100L}\}$, we have
\begin{eqnarray*}
\limsup_{j \rightarrow +\infty}\left(\|\widetilde{\w}_{t_j+1}-\bar{\w}_{t_j}\|^2 + \EE_q[\|\widetilde{\sigma}_{t_j+1}(\cb_d)-\bar{\sigma}_{t_j}(\cb_d)\|_F^2]\right) & = & 0, \\
\limsup_{j \rightarrow +\infty}\left(\|\bar{\w}_{t_j}-\w_{t_j}\|^2 + \EE_q[\|\bar{\sigma}_{t_j}(\cb_d)-\sigma_{t_j}(\cb_d)\|_F^2]\right) & = & 0. 
\end{eqnarray*}
By Lemma~\ref{lemma:boundedness}, the sequence $\{(\widetilde{\w}_{t_j+1}, \widetilde{\sigma}_{t_j+1})\}_{j \geq 0}$ is bounded. This implies that there exists a subsequence $\{(\widetilde{\w}_{t_{j_i}+1}, \widetilde{\sigma}_{t_{j_i}+1})\}_{i \geq 0}$ such that $(\widetilde{\w}_{t_{j_i}+1}, \widetilde{\sigma}_{t_{j_i}+1})$ weakly converges to some point $(\widetilde{\w}, \widetilde{\sigma}) \in \WCal \times \Sigma$, and  
\begin{eqnarray}\label{inequality:PRGIE-main-fourth}
& & \limsup_{j \rightarrow +\infty} \left(- 2\langle \w_\star, \w_{t_j+1} - \w_\star\rangle - 2\EE_q[\langle \sigma_\star(\cb_d), \sigma_{t_j+1}(\cb_d)-\sigma_\star(\cb_d)\rangle]\right) \\
& & = \lim_{i \rightarrow +\infty} \left(- 2\langle \w_\star, \w_{t_{j_i}+1} - \w_\star\rangle - 2\EE_q[\langle \sigma_\star(\cb_d), \sigma_{t_{j_i}+1}(\cb_d)-\sigma_\star(\cb_d)\rangle]\right) \nonumber \\ 
& & = - 2\langle \w_\star, \widetilde{\w} - \w_\star\rangle - 2\EE_q[\langle \sigma_\star(\cb_d), \widetilde{\sigma}(\cb_d)-\sigma_\star(\cb_d)\rangle]. \nonumber
\end{eqnarray}
It is also clear that $(\bar{\w}_{t_{j_i}}, \bar{\sigma}_{t_{j_i}})$ and $(\w_{t_{j_i}}, \sigma_{t_{j_i}})$ both weakly converge to $(\widetilde{\w}, \widetilde{\sigma})$. Recall that the optimality condition of updating $(\widetilde{\w}_{t+1}, \widetilde{\sigma}_{t+1})$ in Eq.~\eqref{opt:PRGIE} is: 
\begin{eqnarray*}
0 & \leq & \EE_q\left[\langle \w - \widetilde{\w}_{t+1}, \widetilde{\w}_{t+1} - \w_t + \gamma\nabla_\w\widehat{\theta}_l(\bar{\w}_t, \bar{\sigma}_t(\cb_d), \cb_l)\rangle\right. \quad \textnormal{for each } (\w, \sigma) \in \WCal \times \Sigma, \\
& & \left. + \langle \sigma(\cb_d) - \widetilde{\sigma}_{t+1}(\cb_d), \widetilde{\sigma}_{t+1}(\cb_d) - \sigma_t(\cb_d) + \gamma\nabla_{\bar{X}}\widehat{\theta}_d(\bar{\w}_t, \bar{\sigma}_t(\cb_d), \cb_d)\rangle\right]. 
\end{eqnarray*}
Equivalently, we have
\begin{eqnarray*}
0 & \leq & \EE_q\left[\langle \w - \widetilde{\w}_{t+1}, \widetilde{\w}_{t+1} - \w_t\rangle + \langle \sigma(\cb_d) - \widetilde{\sigma}_{t+1}(\cb_d), \widetilde{\sigma}_{t+1}(\cb_d) - \sigma_t(\cb_d)\rangle\right] \quad \textnormal{for each } (\w, \sigma) \in \WCal \times \Sigma \\
& & + \gamma \EE_q\left[\langle \w - \bar{\w}_t, \nabla_\w\widehat{\theta}_l(\bar{\w}_t, \bar{\sigma}_t(\cb_d), \cb_l)\rangle + \langle \sigma(\cb_d) - \bar{\sigma}_t(\cb_d), \nabla_{\bar{X}}\widehat{\theta}_d(\bar{\w}_t, \bar{\sigma}_t(\cb_d), \cb_d)\rangle\right] \\ 
& & + \gamma \EE_q\left[\langle \bar{\w}_t - \widetilde{\w}_{t+1}, \nabla_\w\widehat{\theta}_l(\bar{\w}_t, \bar{\sigma}_t(\cb_d), \cb_l)\rangle + \langle \bar{\sigma}_t(\cb_d) - \widetilde{\sigma}_{t+1}(\cb_d), \nabla_{\bar{X}}\widehat{\theta}_d(\bar{\w}_t, \bar{\sigma}_t(\cb_d), \cb_d)\rangle\right]
\end{eqnarray*}
Using the first condition in Assumption~\ref{Assumption:PRG-IE} with $(\w', \sigma')=(\bar{\w}_t, \bar{\sigma}_t)$ and the fact that $\gamma > 0$, we have
\begin{eqnarray*}
\lefteqn{\gamma\EE_q\left[\langle \w - \bar{\w}_t, \nabla_\w\widehat{\theta}_l(\bar{\w}_t, \bar{\sigma}_t(\cb_d), \cb_l)\rangle + \langle \sigma(\cb_d) - \bar{\sigma}_t(\cb_d), \nabla_{\bar{X}}\widehat{\theta}_d(\bar{\w}_t, \bar{\sigma}_t(\cb_d), \cb_d)\rangle\right]} \\ 
& \leq & \gamma\EE_q\left[\langle \w - \bar{\w}_t, \nabla_\w\widehat{\theta}_l(\w_t, \sigma_t(\cb_d), \cb_l)\rangle + \langle \sigma(\cb_d) - \bar{\sigma}_t(\cb_d), \nabla_{\bar{X}}\widehat{\theta}_d(\w_t, \sigma_t(\cb_d), \cb_d)\rangle\right]. 
\end{eqnarray*}
Putting these two inequalities together with $t=t_{j_i}$ yields that, for each $(\w, \sigma) \in \WCal \times \Sigma$,
\begin{eqnarray*}
0 & \leq & \EE_q\left[\langle \w - \widetilde{\w}_{t_{j_i}+1}, \widetilde{\w}_{t_{j_i}+1} - \w_{t_{j_i}}\rangle + \langle \sigma(\cb_d) - \widetilde{\sigma}_{t_{j_i}+1}(\cb_d), \widetilde{\sigma}_{t_{j_i}+1}(\cb_d) - \sigma_{t_{j_i}}(\cb_d)\rangle\right] \\
& & + \gamma \EE_q\left[\langle \w - \bar{\w}_{t_{j_i}}, \nabla_\w\widehat{\theta}_l(\w, \sigma(\cb_d), \cb_l)\rangle + \langle \sigma(\cb_d) - \bar{\sigma}_{t_{j_i}}(\cb_d), \nabla_{\bar{X}}\widehat{\theta}_d(\w, \sigma(\cb_d), \cb_d)\rangle\right] \\ 
& & + \gamma \EE_q\left[\langle \bar{\w}_{t_{j_i}} - \widetilde{\w}_{t_{j_i}+1}, \nabla_\w\widehat{\theta}_l(\bar{\w}_{t_{j_i}}, \bar{\sigma}_{t_{j_i}}(\cb_d), \cb_l)\rangle + \langle \bar{\sigma}_{t_{j_i}}(\cb_d) - \widetilde{\sigma}_{t_{j_i}+1}(\cb_d), \nabla_{\bar{X}}\widehat{\theta}_d(\bar{\w}_{t_{j_i}}, \bar{\sigma}_{t_{j_i}}(\cb_d), \cb_d)\rangle\right]. 
\end{eqnarray*}
Letting $i \rightarrow +\infty$ in the above inequality, for each $(\w, \sigma) \in \WCal \times \Sigma$, we have
\begin{equation}\label{inequality:PRGIE-main-fifth}
\EE_q\left[\left\langle \w-\widetilde{\w}, \nabla_\w\widehat{\theta}_l(\w, \sigma(\cb_d), \cb_l)\right\rangle + \left\langle \sigma(\cb_d) - \widetilde{\sigma}(\cb_d), \nabla_{\bar{X}}\widehat{\theta}_d(\w, \sigma(\cb_d), \cb_d)\right\rangle\right] \ \geq \ 0, 
\end{equation}
Using Lemma~\ref{lemma:minty} and Eq.~\eqref{inequality:PRGIE-main-fifth}, the point $(\widetilde{\w}, \widetilde{\sigma})$ is the solution of the VI in Eq.~\eqref{prob:VI-infinite}. Under Assumption~\ref{Assumption:main} and~\ref{Assumption:PRG-IE}, the VI in Eq.~\eqref{prob:VI-infinite} has a unique solution. Thus, $(\widetilde{\w}, \widetilde{\sigma}) = (\w_\star, \sigma_\star)$ is a unique Bayesian equilibrium. 

Finally, we consider Lemma~\ref{lemma:sequence} with $s_t = r_t$, $a_t=3\delta_t/4$ and 
\begin{equation*}
b_t \ = \ 4\delta_t M_1 - 2\langle \w_\star, \w_{t+1} - \w_\star\rangle - 2\EE_q[\langle \sigma_\star(\cb_d), \sigma_{t+1}(\cb_d)-\sigma_\star(\cb_d)\rangle]. 
\end{equation*}
More specifically, we have (i) the sequence $\{s_t\}_{t \geq 0}$ is nonnegative and $\{a_t\}_{t \geq 0}$ is a sequence in $(0, 1)$ satisfying $\sum_{t=0}^{+\infty} a_t = +\infty$; (ii) Eq.~\eqref{inequality:PRGIE-second} implies that $s_{t+1} \leq (1-a_t)s_t + a_tb_t$ for all $t \geq 2$; (iii) for every subsequence $\{s_{t_j}\}_{j \geq 0}$ of $\{s_t\}_{t \geq 0}$ satisfying that $\liminf_{j \rightarrow +\infty} (s_{t_j+1} - s_{t_j}) \geq 0$, Eq.~\eqref{inequality:PRGIE-main-fourth} and the fact that $\delta_t \rightarrow 0$ implies that
\begin{equation*}
\limsup_{j \rightarrow +\infty} b_{t_j} \ = \ \limsup_{j \rightarrow +\infty} \left(\delta_{t_j} M_1 - 2\langle \w_\star, \w_{t_j+1} - \w_\star\rangle - 2\EE_q[\langle \sigma_\star(\cb_d), \sigma_{t_j+1}(\cb_d)-\sigma_\star(\cb_d)\rangle]\right) \ = \ 0. 
\end{equation*}
Therefore, we conclude that $r_t \rightarrow 0$ as $t \rightarrow +\infty$. This completes the proof. 

\section{Postponed Proofs in Section~\ref{sec:alg_rand}}
In this section, we provide the detailed proof for Theorem~\ref{Theorem:PGRBC}. We start by reviewing one preliminary result in the literature which is a fact of sequences first established in~\cite{Chung-1954-Stochastic} (although it does not appear to be widely known). 
\begin{lemma}\label{lemma:Chung}
Let $\{a_t\}_{t \geq 0}$ be a non-negative sequence such that 
\begin{equation*}
a_{t+1} \ \leq \ \left(1-\frac{P}{t^p}\right)a_t + \frac{Q}{t^{p+q}}, 
\end{equation*}
where $P > q > 0$, $0 < p \leq 1$ and $Q > 0$. Then:
\begin{equation*}
a_t \ \leq \ \left\{\begin{array}{cl}
\frac{Q}{P}\frac{1}{t^q}, & \textnormal{if } 0 < p < 1, \\
\frac{Q}{P-q}\frac{1}{t^q}, & \textnormal{if } p = 1. 
\end{array}\right.  
\end{equation*}
\end{lemma}
\textbf{Proof of Theorem~\ref{Theorem:PGRBC}.} By Corollary~\ref{corollary:VI-existence-third}, the Bayesian regression game $G = (\WCal, \Sigma, \widehat{\theta}_l, \widehat{\theta}_d, \cb_l, q)$ has a unique Bayesian equilibrium $(\w_\star, \sigma_\star^1, \ldots, \sigma_\star^K)$ under Assumption~\ref{Assumption:main} and~\ref{Assumption:PG-RBC}. Define $E_t=\|\w_{t+1}-\w_\star\|^2+\sum_{k=1}^K \|\sigma_{t+1}^k-\sigma_\star^k\|_F^2$, we derive from the update formula in Algorithm~\ref{Algorithm:PG-RBC} and the fact that two orthogonal projection mappings $P_\WCal$ and $P_\XCal$ are nonexpansive that
\begin{eqnarray*}
& & \EE[E_{t+1} \mid (\w_t, \sigma_t^1, \ldots, \sigma_t^K)] \\
& & = \sum_{k=1}^K p_k \left[\|P_\WCal(\w_t - \gamma_t\nabla_\w\widehat{\theta}_l(\w_t, \sigma_t^k, \cb_l)) - \w_\star\|^2 + \|P_\XCal(\sigma_t^k - \gamma_t\nabla_{\bar{X}}\widehat{\theta}_d(\w_t, \sigma_t^k, \sv_k)) - \sigma_\star^k\|_F^2\right] \\ 
& & + \sum_{k=1}^K p_k \left[\sum_{j \neq k} \|\sigma_t^k - \sigma_\star^k\|_F^2 \right] \\
& &\leq  \sum_{k=1}^K p_k \left[\|\w_t - \gamma_t\nabla_\w\widehat{\theta}_l(\w_t, \sigma_t^k, \cb_l) - \w_\star\|^2 + \|\sigma_t^k - \gamma_t\nabla_{\bar{X}}\widehat{\theta}_d(\w_t, \sigma_t^k, \sv_k) - \sigma_\star^k\|_F^2 + \sum_{j \neq k} \|\sigma_t^k - \sigma_\star^k\|_F^2\right]. 
\end{eqnarray*}
Using the second condition in Assumption~\ref{Assumption:PG-RBC}, we have
\begin{eqnarray*}
\|\w_t - \gamma_t\nabla_\w\widehat{\theta}_l(\w_t, \sigma_t^k, \cb_l) - \w_\star\|^2 & \leq & \|\w_t - \w_\star\|^2 - 2\gamma_t\langle \w_t - \w_\star, \nabla_\w\widehat{\theta}_l(\w_t, \sigma_t^k, \cb_l)\rangle + \gamma_t^2 G^2, \\
\|\sigma_t^k - \gamma_t\nabla_{\bar{X}}\widehat{\theta}_d(\w_t, \sigma_t^k, \sv_k) - \sigma_\star^k\|_F^2 & \leq & \|\sigma_t^k - \sigma_\star^k\|_F^2 - 2\gamma_t\langle \sigma_t^k - \sigma_\star^k, \nabla_{\bar{X}}\widehat{\theta}_d(\w_t, \sigma_t^k, \sv_k)\rangle + \gamma_t^2 G^2. 
\end{eqnarray*}
Putting these pieces together with the fact that $\sum_{k=1}^K p_k = 1$ yields that 
\begin{equation}\label{inequality:PGRBC-first}
\EE\left[E_{t+1} \mid (\w_t, \sigma_t^1, \ldots, \sigma_t^K)\right] \leq E_t + 2\gamma_t^2 G^2 - 2\gamma_t\sum_{k=1}^K p_k\left[\langle \w_t - \w_\star, \nabla_\w\widehat{\theta}_l(\w_t, \sigma_t^k, \cb_l)\rangle + \langle \sigma_t^k - \sigma_\star^k, \nabla_{\bar{X}}\widehat{\theta}_d(\w_t, \sigma_t^k, \sv_k)\rangle \right].
\end{equation}
Since the point $(\w_\star, \sigma_\star^1, \ldots, \sigma_\star^K)$ is a Bayesian equilibrium, we have
\begin{equation}\label{inequality:PGRBC-second}
\sum_{k=1}^K p_k\left[\left\langle \w_t-\w_\star, \nabla_\w\widehat{\theta}_l(\w_\star, \sigma_\star^k, \cb_l)\right\rangle + \left\langle \sigma_t^k - \sigma_\star^k, \nabla_{\bar{X}}\widehat{\theta}_d(\w_\star, \sigma_\star^k, \sv_k)\right\rangle\right] \ \geq \ 0,
\end{equation}
Summing up Eq.~\eqref{inequality:PGRBC-first} and Eq.~\eqref{inequality:PGRBC-second} and using the first condition in Assumption~\ref{Assumption:PG-RBC}, we have
\begin{equation*}
\EE[E_{t+1} \mid (\w_t, \sigma_t^1, \ldots, \sigma_t^K)] \ \leq \ (1-2\lambda\gamma_t)E_t + 2\gamma_t^2 G^2.
\end{equation*}
Taking the expectation of both sides and using the definition of $\gamma_t$, we have
\begin{equation*}
\EE[E_{t+1}] \ \leq \ (1-2\lambda\gamma_t)\EE[E_t] + 2\gamma_t^2 G^2 \ = \ \left(1-\frac{2\lambda\gamma_0}{t}\right)\EE[E_t] + \frac{2\gamma_0^2 G^2}{t^2}, \quad \textnormal{for all } t \geq 1. 
\end{equation*}
Applying Lemma~\ref{lemma:Chung} with $P=2\lambda\gamma_0 > 1$, $Q=2\gamma_0^2 G^2$ and $p=q=1$, we have
\begin{equation*}
\EE[E_t] \ \leq \ \frac{2\gamma_0^2 G^2}{2\lambda\gamma_0-1}\frac{1}{t} \ = \ O\left(\frac{1}{t}\right). 
\end{equation*}
This completes the proof.

\end{document}